\documentclass[10pt,letterpaper]{uconn-minds}

\usepackage{amsmath,amsfonts,bm}

\def\eqref#1{equation~\ref{#1}}
\def\1{\bm{1}}

\DeclareMathAlphabet{\mathsfit}{\encodingdefault}{\sfdefault}{m}{sl}
\SetMathAlphabet{\mathsfit}{bold}{\encodingdefault}{\sfdefault}{bx}{n}

\newcommand{\E}{\mathbb{E}}

\newcommand{\KL}{D_{\mathrm{KL}}}
\newcommand{\Var}{\mathrm{Var}}

\DeclareMathOperator*{\argmax}{arg\,max}

\usepackage{algorithm}
\usepackage{multirow}
\usepackage{tabularx}
\usepackage{wrapfig}
\newcommand{\alignedwrapgraphic}[2][]{%
  \raisebox{\dimexpr\ht\strutbox-\height\relax}{%
    \includegraphics[width=\linewidth,#1]{#2}}}
\usepackage{needspace}
\usepackage{bm}
\newcommand{\ms}[2]{#1{\scriptsize$\pm$#2}}   % mean +- std
\definecolor{dndelta}{RGB}{176,52,52}
\definecolor{updelta}{RGB}{38,115,68}
\colorlet{algorithmink}{uconnBlue}
\colorlet{algorithmwash}{abstractBlue}

\usepackage{xspace}
\newcommand{\methodname}{\textsc{CLEAT}\xspace}
\graphicspath{{figures/}}
\usepackage{amsmath,amssymb,amsthm}

\newcommand{\PASS}{\textsc{pass}}
\newcommand{\QUERY}{\textsc{query}}
\newcommand{\RQV}{\textsc{RQV}}

\newcommand{\IG}{\mathrm{IG}}

\colorlet{thmbg}{abstractBlue}
\colorlet{thmink}{uconnBlue}
\colorlet{ourrow}{abstractBlue}
\tcbset{mindsbox/.style={
  enhanced jigsaw, frame hidden,
  colback=thmbg, coltext=paperText,
  colbacktitle=thmbg, coltitle=thmink,
  boxrule=0pt, titlerule=0pt, arc=10pt, outer arc=10pt,
  left=9pt, right=9pt, top=6pt, bottom=6pt,
  before skip=9pt, after skip=9pt},
  thmbox/.style={mindsbox, breakable}}

\newtcolorbox{promptbox}[1]{
  mindsbox, breakable,
  fonttitle=\small\normalfont\bfseries, title={#1},
  fontupper=\small\ttfamily, fontlower=\small\ttfamily,
  before upper={\raggedright\setlength{\parindent}{0pt}\setlength{\parskip}{4pt}},
  before lower={\raggedright\setlength{\parindent}{0pt}\setlength{\parskip}{4pt}},
  toptitle=5pt, bottomtitle=5pt,
  segmentation style={solid, draw=uconnBlue!20!white, line width=0.3pt}}
\newcommand{\promptslot}[1]{{\color{thmink}\{#1\}}}

\newtcolorbox{casebox}[1]{
  mindsbox,
  fonttitle=\small\normalfont\bfseries, title={#1},
  fontupper=\small,
  before upper={\setlength{\parindent}{0pt}\setlength{\parskip}{5pt}},
  toptitle=5pt, bottomtitle=5pt}

\newtheoremstyle{boxthm}%
  {0pt}{0pt}%                       space above / below
  {\normalfont}%                    body font
  {0pt}%                            indent
  {\bfseries\color{thmink}}%
  {.}%                              punctuation after head
  {.5em}%                           space after head
  {\thmname{#1}\thmnumber{ #2}\thmnote{ \bfseries(#3)}}
\theoremstyle{boxthm}
\newtheorem{proposition}{Proposition}

\tcolorboxenvironment{proposition}{thmbox}

\title{CLEAT: Learning Value-Guided Clarification\\for Interactive Agents}

\author{
    Jiatan Huang\textsuperscript{1}\quad
    Zheyuan Zhang\textsuperscript{2}\quad
    Ziming Li\textsuperscript{1}\quad
    Tianyi Ma\textsuperscript{2} \\
    Mingchen Li\textsuperscript{3}\quad
    Yanfang Ye\textsuperscript{2}\quad
    Chuxu Zhang\textsuperscript{1,$\dagger$}
}
\affiliation{
    \textsuperscript{1}University of Connecticut\quad
    \textsuperscript{2}University of Notre Dame \\
    \textsuperscript{3}University of Massachusetts, Amherst \\[2pt]
    \texttt{\{jiatan.huang,chuxu.zhang\}@uconn.edu}\quad
    \textsuperscript{$\dagger$}Corresponding author
}

\paperabstract{LLM agents can execute increasingly complex tasks, but effective assistance
also requires aligning their actions with goals and constraints that users
may leave unstated. We study this problem as \emph{interactive intent
alignment} through clarification.
Yet clarification is useful only when the agent can turn the answer into
better decisions. As training changes how the agent handles incomplete
information and uses additional answers, the same clarification can become
more or less useful.
We introduce
\textbf{\methodname} (\textbf{CL}arification through valu\textbf{E}-guided
\textbf{A}gent co-\textbf{T}raining), a framework that co-trains a
clarification-value controller and the task policy from shared interaction
outcomes. We formulate \emph{residual query value} to measure the task
utility gained by asking rather than proceeding under the current policy,
and estimate it through paired rollouts from the same interaction history.
These paired returns train the controller to predict clarification value
and the task policy to select useful questions and act on the answers. 
The controller guides subsequent rollout collection, while outcomes
from the updated agent provide new value labels that allow clarification
guidance to adapt as the agent learns.
With Qwen3-4B, \methodname
achieves the highest mean score in all ten evaluation settings across four
benchmarks. Gains over the strongest baseline reach $13.2$ and $6.5$
percentage points on UserGym Travel and IN3 underspecification judgment
accuracy, respectively. Ablations show that co-training improves the task policy even
when the controller is removed at inference, while adapting clarification
guidance to the evolving policy yields further gains.
}
\hypersetup{
    pdftitle={CLEAT: Learning Value-Guided Clarification for Interactive Agents},
    pdfauthor={Jiatan Huang, Zheyuan Zhang, Ziming Li, Tianyi Ma, Mingchen Li, Yanfang Ye, Chuxu Zhang}}

\begin{document}

\addtocontents{toc}{\protect\setcounter{tocdepth}{-1}}

\maketitle
\section{Introduction}
\label{sec:introduction}

Large language model (LLM) agents have made rapid progress in planning and
tool use, enabling them to carry out increasingly complex
tasks~\citep{zhou2024webarenarealisticwebenvironment,yang2024sweagentagentcomputerinterfacesenable,li2026graph}.
However, stronger execution alone does not ensure that an agent's actions
reflect what the user actually wants. In practice, users may leave important
constraints unstated, rely on mistaken assumptions, or struggle to articulate
their preferences~\citep{qian2024tellmoreimplicituser,vijayvargiya2026ambigsweinteractiveagentsovercome},
creating a gap between their explicit requests and intended goals.
Effective assistance therefore requires agents to help close this gap
through interaction, rather than expect users to specify everything
upfront. We study this problem as \emph{interactive intent alignment},
where agents align their task decisions with the user's intended goals
and constraints through interaction.

Clarification can support this alignment by resolving uncertainty about the
user's goals and constraints.\footnote{We use \emph{clarification} broadly
to refer to questions an agent asks before acting, whether to fill in
missing information, resolve ambiguity, or challenge a questionable
premise.} Yet clarification can reveal more about the user without
improving the agent's decisions, while still requiring the user's time
and attention. For example, suppose an agent must choose between two
flights operated by the same airline but arriving at different times.
Asking about airline preference may reveal more about the user, but it
does not help distinguish between the two options. By contrast, asking
about an unstated arrival deadline can determine which flight meets the
user's needs (Figure~\ref{fig:intent-alignment}). The challenge is therefore
to recognize when clarification can improve the task outcome and what
information the agent needs to obtain.

\Needspace{18\baselineskip}
\begingroup
\setlength{\intextsep}{0pt}
\begin{wrapfigure}{r}{0.43\textwidth}
    \centering
    \setlength{\parskip}{0pt}
    \alignedwrapgraphic[trim=2bp 2.5bp 0.75bp 1.25bp,clip]{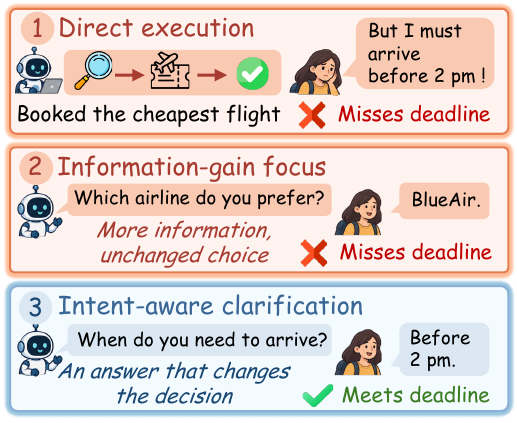}
    \par
    \caption{\textbf{Interactive intent alignment.} \methodname targets
    questions that improve task outcomes, not merely reduce uncertainty.}
    \label{fig:intent-alignment}
\end{wrapfigure}

Prior work guides clarification through information-based
signals~\citep{deng2026uncertaintyawareclarificationllmagents,kong2026infopoinformationdrivenpolicyoptimization},
value-of-information
estimates~\citep{rao2018learningaskgoodquestions,dong2026valueinformationframeworkhumanagent},
and rewards from subsequent interaction
outcomes~\citep{zhang2025modelingfutureconversationturns,wu2025collabllmpassiverespondersactive}.
Despite this progress, identifying which clarifications help the agent
remains challenging, particularly while its policy is being trained.
\textbf{First, a clarification can appear useful without actually
improving the task outcome.} A question may reveal additional information
or be followed by a successful outcome, yet the same agent might have
achieved an equally good result by proceeding on its own.
\textbf{Second, clarification value can change even at a fixed
interaction history.} Training can improve how the agent uses an answer
and how it proceeds on its own by different amounts. The incremental
value of the same clarification can therefore change as the policy
learns, requiring renewed value supervision.

To address these challenges, we introduce \methodname, which couples
clarification-value learning with task-policy learning through paired
interaction outcomes. For each evaluated clarification target, paired
rollouts compare asking with native continuation under the current policy
from the same interaction history. Native continuation retains the agent's
own ability to ask, so the utility difference measures the added benefit
of the targeted clarification. We call this difference the
\emph{residual query value} (RQV).
The same outcomes provide controller value targets and policy advantages;
for each target, the mean query advantage equals its RQV label. During training,
controller predictions prioritize which comparisons to collect, while
sampling-corrected policy updates use observed task returns. Fresh outcomes
renew value supervision as the agent learns. At inference, the trained
agent acts on clarification answers, while the controller guides when
and what to ask using predicted RQV.
\par
\WFclear
\endgroup

We evaluate \methodname on four benchmarks spanning personalized decision
making, user-in-the-loop tool use, and requests with missing information or
false premises. With Qwen3-4B as the backbone, \methodname achieves the
highest mean score in all ten evaluation settings. Gains over the strongest
baseline reach $13.2$ and $6.5$ percentage points on UserGym Travel and
IN3 underspecification judgment accuracy, respectively.
Ablations show that the trained policy improves over the original
backbone even without controller-based action selection at inference.
Updating the controller during policy training also yields higher mean
scores than freezing it in all six ablation settings.

Our contributions are summarized as follows.
\begingroup
\setlength{\topsep}{3pt}
\setlength{\partopsep}{0pt}
\setlength{\itemsep}{2pt}
\setlength{\parsep}{0pt}
\begin{itemize}
    \item We construct \textbf{paired policy-relative supervision} from
    matched ask-versus-proceed outcomes, providing controller value targets
    and task-policy advantages relative to the agent's native continuation.
    \item We develop \textbf{value-guided co-training}, in which predicted
    values guide rollout collection, observed returns drive
    sampling-corrected policy updates, and fresh outcomes renew value
    supervision as the agent learns.
    \item We demonstrate \textbf{gains beyond inference-time guidance},
    with ablations showing that co-training improves the task policy
    itself and that updating clarification values provides benefits
    beyond fixed guidance.
\end{itemize}
\endgroup

\section{Preliminaries}
\label{sec:prelim}

\subsection{Interactive Intent Alignment}
\label{sec:prelim-setup}

We consider interactive tasks in which the user's initial request may
not fully or accurately express the goals and constraints that determine
success. For example, a travel assistant may need to establish whether
the user prioritizes price or ticket flexibility, while a service agent
may need to identify which reservation the user wants to change.
Clarification can help resolve these gaps, but its purpose is to inform
task decisions rather than serve as an objective in itself.

Formally, let $G$ denote the user's latent goal and $h$ the interaction
history, which provides only partial information about $G$. The task
policy $\pi_\theta$ selects actions based on $h$ and may ask for
clarification when permitted by the environment, with replies added to
$h$. We evaluate the resulting episode $\tau$ by its task
utility $U(\tau,G)$, including any benchmark-specified interaction costs.
To compare asking with proceeding, we consider their
\emph{continuations} from $h$, each comprising the remainder of
an episode generated under $\pi_\theta$. During training, we assume
access to an environment that can sample multiple continuations from
the same $h$ and $G$, with the policy held fixed within each comparison.

\subsection{Group-Relative Policy Optimization}
\label{sec:prelim-grpo}

To train the task policy from these sampled outcomes, we build on
Group Relative Policy Optimization
(GRPO)~\citep{shao2024deepseekmathpushinglimitsmathematical}.
GRPO estimates advantages by normalizing returns within a group,
without a learned critic. In our setting, consider $n$ continuations
sampled under the matched conditions above, yielding completed
episodes $\tau_1,\ldots,\tau_n$ with utilities $U_i=U(\tau_i,G)$.
Standard GRPO constructs the group-relative advantages as
\begin{equation}
A_i=\frac{U_i-\bar U}{s_U+\delta},
\qquad
\bar U=\frac{1}{n}\sum_{j=1}^{n}U_j,
\label{eq:grpo}
\end{equation}
where $s_U$ is the standard deviation of the group utilities and
$\delta>0$ ensures numerical stability. These advantages are used
in a clipped policy-ratio objective with KL regularization toward
a reference policy. Our method retains policy-ratio clipping and
reference-policy regularization while modifying how rollouts are
collected and how their utilities are used to construct advantages.
Implementation details are provided in
Appendix~\ref{app:implementation}.

\section{\methodname}
\label{sec:method}

\methodname co-trains a task policy and a clarification-value controller
from paired rollouts. At each training history, we compare asking about
selected targets with letting the current policy continue on its own.
The observed returns provide value labels for the controller and
advantages for updating the task policy. The updated controller guides
target selection in the next round, and new rollouts under the updated
policy provide fresh value labels (Figure~\ref{fig:method}).

% Put the overview at the next page top, after its introduction in the text.
\suppressfloats[t]
\begin{figure}[t]
    \centering
    \setlength{\abovecaptionskip}{3pt}
    \setlength{\belowcaptionskip}{0pt}
    \setlength{\parskip}{0pt}
    \includegraphics[width=\linewidth]{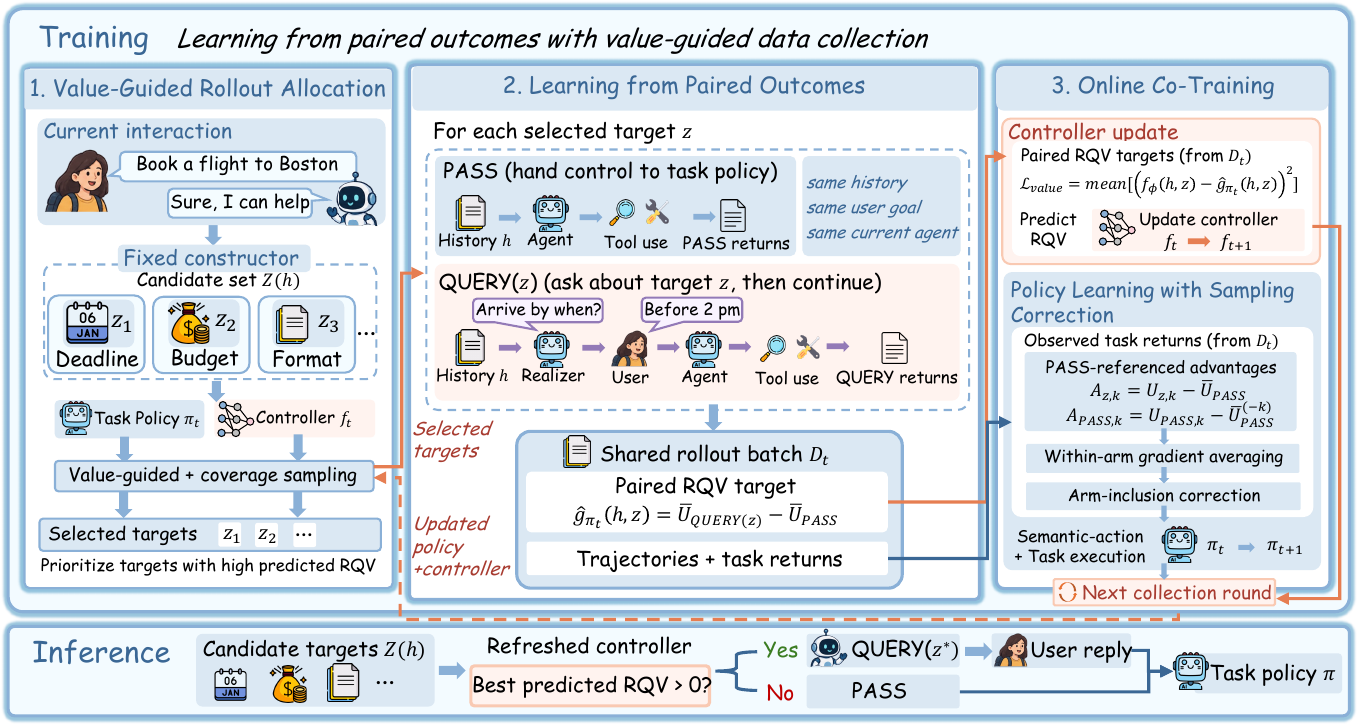}
    \par\vspace{-2pt}
    \caption{\textbf{Overview of \methodname.} Value-guided paired rollouts
    train the controller and task policy from shared outcomes. At inference,
    the controller gates clarification by predicted incremental utility.}
    \label{fig:method}
    \vspace{-12pt}
\end{figure}

\subsection{Problem Setup and Residual Query Value}
\label{sec:setup}
\label{sec:ig}

At an interaction history $h$, we compare letting the task policy proceed
($\PASS$) with first asking for additional information ($\QUERY(z)$).
Here $z$ identifies what information to request, such as a travel date or
budget. Choosing $\PASS$ delegates the remaining interaction to the task
policy, which retains its own ability to ask questions.

A fixed target constructor supplies the finite candidate set $\mathcal Z(h)$.
A history is \emph{clarification-eligible} when another user exchange is
permitted and this set is nonempty. For $\QUERY(z)$, a fixed question realizer turns
the target into a question, and the question--answer exchange is appended
to $h$ before task execution resumes. This separates the decision of what
to ask from question wording.

For a fixed task policy $\pi$, let $V^\pi(h)$ denote the expected utility
of choosing $\PASS$ and $Q^\pi(h,z)$ that of choosing $\QUERY(z)$. Both
continuations use the same policy $\pi$ after the initial action. Their
difference defines the \emph{residual query value} (\RQV),
\begin{equation}
g_\pi(h,z)=Q^\pi(h,z)-V^\pi(h).
\label{eq:rqv}
\end{equation}
The expectations average over the latent goal $G\mid h$ and rollout
randomness. RQV measures what a targeted clarification adds beyond native
continuation, including any questions the agent would ask on its own.
Benchmark-specified interaction costs are included in utility.

\Needspace{2\baselineskip}
Information gain, $\IG(h,z)=I(G;Y_z\mid h)$, measures what an answer $Y_z$
reveals about the user's goal. An informative answer need not improve the
agent's decisions, so IG and RQV can rank clarification targets differently.
This reflects the classical distinction between information gain and decision
value~\citep{howard1966information}.

RQV also depends on how the policy uses the answer. Ignoring it may leave
utility unchanged, while misusing it can hurt even without interaction
costs. Classical value of information evaluates an optimal decision maker;
RQV evaluates the policy being trained.
Training can improve answer use and native continuation by different
amounts, changing RQV even for a fixed history and target. We therefore
renew value supervision as the policy learns (Appendix~\ref{app:proofs}).

\subsection{Learning from Paired Outcomes}
\label{sec:paired}

We estimate RQV by comparing $\QUERY(z)$ and $\PASS$ continuations from the
same history $h$ and latent user goal $G$. We hold the task policy $\pi$
fixed throughout the comparison. After the initial action, both branches
follow $\pi$ without further controller intervention. Writing $\bar U_a$
for the mean task utility of continuations collected for action $a$, the
paired estimate is
\begin{equation}
\widehat g_\pi(h,z)
=
\bar U_{\QUERY(z)}-\bar U_{\PASS}.
\label{eq:rqv-estimator}
\end{equation}
Averaging these labels over rollout randomness and $G\mid h$ recovers
$g_\pi(h,z)$ (Appendix~\ref{app:interface}).
These measured gains provide the training labels for a lightweight controller
$f_\phi(h,z)$. The controller predicts RQV from the history and candidate
target, without observing the latent goal. At training step $t$, we fit it
to labels collected under $\pi_t$ by minimizing
\begin{equation}
\mathcal L_{\mathrm{value}}(\phi;\mathcal D_t)
=
\frac{1}{|\mathcal D_t|}
\sum_{(h,z)\in\mathcal D_t}
\left(f_\phi(h,z)-\widehat g_{\pi_t}(h,z)\right)^2,
\label{eq:controller-loss}
\end{equation}
where $\mathcal D_t$ contains the current batch of history--target pairs and
their labels. Once fitted, the controller scores candidates without running
paired continuations for each one. We use these predictions to allocate
subsequent rollouts.

\subsection{Value-Guided Rollout Allocation}
\label{sec:controller}

Under a fixed rollout budget, sampling only from the task policy may
overlook useful low-probability targets. We therefore compare a value-guided
target $z_g$ and, when available, a distinct coverage target $z_c$ against a
shared $\PASS$ reference. The guided draw prioritizes predicted RQV, while
the coverage draw gathers evidence about candidates the controller may
undervalue.

The task policy provides the starting distribution for both draws. It
assigns probabilities $\pi_\theta^A(a\mid h)$ over $\PASS$ and the available
$\QUERY(z)$ actions by scoring fixed text representations of these actions
and normalizing across them, without an additional action head.
Conditioning on asking gives the target distribution $p_\pi(z\mid h)$.
We mix it with a uniform component to form a reference that keeps every
candidate reachable,
\[
\bar p_\pi(z\mid h)
=
(1-\epsilon)p_\pi(z\mid h)+\frac{\epsilon}{|\mathcal Z(h)|},
\qquad 0<\epsilon<1.
\]
For the guided target $z_g$, we choose a sampling distribution $\mu$ that
balances predicted RQV against deviation from the reference $\bar p_\pi$,
\[
\max_{\mu\in\Delta(\mathcal Z(h))}
\E_{z\sim\mu}[f_\phi(h,z)]
-\beta^{-1}\KL(\mu\Vert\bar p_\pi).
\]
The resulting proposal is
\begin{equation}
\mu^*(z\mid h)
\propto
\bar p_\pi(z\mid h)
\exp\{\beta f_\phi(h,z)\},
\label{eq:guided-proposal}
\end{equation}
where $\beta>0$ controls the strength of value guidance.
We sample $z_g\sim\mu^*(\cdot\mid h)$.

For the coverage target $z_c$, we sample from $\bar p_\pi(\cdot\mid h)$
renormalized over the remaining candidates $\mathcal Z(h)\setminus\{z_g\}$.
This draw reserves rollout budget for another candidate without value
guidance.
At each history $h$, we collect $K_p\geq2$ $\PASS$ continuations,
$K_g\geq1$ continuations for $\QUERY(z_g)$, and $K_c\geq1$ continuations
for $\QUERY(z_c)$.
Each distinct semantic action is an arm. With only one candidate, its
query arm receives $K_g+K_c$ continuations. The arms and replica counts are
chosen before observing continuation outcomes (Appendix~\ref{app:rollouts}).

\subsection{Policy Learning with Sampling Correction}
\label{sec:rollouts}

Controller predictions guide collection; observed returns provide policy
advantages. We correct for nonuniform arm inclusion to preserve the
collecting policy's action weighting at a fixed history.

\paragraph{Task-return advantages.}
Replacing the group-mean baseline in Equation~\ref{eq:grpo} with $\PASS$
references gives the query and $\PASS$ advantages,
\[
A_{z,k}=U_{z,k}-\bar U_{\PASS},
\qquad
A_{\PASS,k}=U_{\PASS,k}-\bar U_{\PASS}^{(-k)},
\]
where $U_{z,k}$ and $U_{\PASS,k}$ are the observed utilities of the $k$th
continuation in the corresponding arms. The baseline
$\bar U_{\PASS}^{(-k)}$ is the mean return of the other $\PASS$
continuations. This leave-one-out reference prevents a $\PASS$ trajectory's
own return from entering its baseline. We treat advantages as fixed during
optimization and do not apply group standard-deviation normalization.

For each query arm, averaging these advantages gives exactly the RQV label
in Equation~\ref{eq:rqv-estimator}. Thus, the controller learns the mean gain
from asking about a target, while the policy learns from the return of each
individual continuation. Both use observed outcomes; the controller's
predictions serve only to guide collection during training.

\paragraph{Correcting arm selection.}
Guided collection changes both how likely an arm is to appear and how many
continuations it receives. We first average gradient contributions within
each arm, then weight each arm mean
to recover the collecting policy's action probabilities.
At a fixed history $h$, let $\mathcal I_h$ be the set of included arms and
$\rho_h(a)$ the probability that arm $a$ is included.
A query arm is included when selected by either the guided or the
coverage draw, while $\rho_h(\PASS)=1$.
\begingroup
\postdisplaypenalty=10000 % Keep the estimator with its definition below.
With collecting policy $\pi_\theta$, the Horvitz--Thompson correction is
\begin{equation}
\widehat H(h)
=
\sum_{a\in\mathcal{I}_h}
\frac{\pi_\theta^A(a\mid h)}{\rho_h(a)}
\,\bar F_a,
\label{eq:ht-estimator}
\end{equation}
where $\bar F_a$ is the mean policy-gradient contribution for arm $a$.
Each continuation contributes its advantage multiplied by the gradient of
the task policy's log-probability. This log-probability covers the initial semantic
action and subsequent agent-generated tokens. We hold the correction
weights fixed when forming the gradient estimate.
\par
\endgroup

Uniform mixing and the always-included $\PASS$ arm ensure positive
inclusion probabilities. Conditional on $h$, $G$, and the selected arms, continuations
are sampled independently under the collecting policy.
Under this conditioning, each reference is independent of the
continuation it evaluates, and the shared and leave-one-out $\PASS$
references have the same expectation across arms.
The baseline terms therefore cancel in expectation under on-policy
arm weighting.

\Needspace{15\baselineskip}
\begin{proposition}[Policy Gradient Preservation under Guided Sampling]
\label{prop:selection}
Fix a clarification-eligible history $h$ and the collecting policy
$\pi_\theta$. Let $m_a(h)$ denote the expected gradient contribution of
arm $a$ under on-policy continuation sampling. Suppose every arm with
$\pi_\theta^A(a\mid h)>0$ has positive inclusion probability and
\[
\E[\bar F_a\mid h,\ a\in\mathcal I_h]=m_a(h).
\]
Then the corrected estimator satisfies
\[
\E[\widehat H(h)\mid h]
=
\sum_a\pi_\theta^A(a\mid h)\,m_a(h),
\]
recovering the expected gradient contribution under on-policy arm sampling
at $h$.
\end{proposition}

At a fixed history and collecting policy, the correction preserves the
expected unclipped task-return gradient. Guidance can still affect
finite-budget learning by changing estimator variance and the controller's
training distribution. This guarantee excludes policy-ratio clipping and
reference-policy regularization (Appendix~\ref{app:selection-proof}).

\subsection{Online Co-Training and Inference}
\label{sec:cotrain}

\paragraph{Online co-training.}
We repeat collection and update both components to track how clarification
value changes as the task policy learns. First, we bootstrap $f_0$ on paired
rollouts under $\pi_0$, drawing targets from $\bar p_{\pi_0}$ without value
guidance. At step $t$, we hold both components fixed during collection.
$f_t$ guides target selection, and $\pi_t$ generates the continuations.
We then use the same batch to fit $f_{t+1}$ to the paired value labels and
update $\pi_t$ to $\pi_{t+1}$ using the sampling-corrected policy update.

Sharing the batch requires only one rollout phase for the two updates.
It also introduces a one-update lag because $f_{t+1}$ learns from outcomes under
$\pi_t$ and then guides $\pi_{t+1}$. After the final policy update, we hold
$\pi_T$ fixed and collect fresh paired rollouts to refresh the controller
with labels generated under the final task policy.
Algorithm~\ref{alg:rqpo} summarizes the procedure, and
Appendix~\ref{app:staleness} discusses the supervision lag.

\paragraph{Inference.}
At inference, we retain the target constructor and question realizer.
While active, the refreshed controller asks about the highest-scoring target
at each clarification-eligible history $h$ only if its predicted RQV is
positive,
\begin{equation}
z^\star=\argmax_{z\in\mathcal{Z}(h)}f_\phi(h,z),
\qquad
a=
\begin{cases}
\QUERY(z^\star), & f_\phi(h,z^\star)>0,\\
\PASS, & \text{otherwise}.
\end{cases}
\label{eq:inference-policy}
\end{equation}
After a query, the controller repeats this rule at the next eligible history.
Choosing $\PASS$ ends controller intervention for the rest of the interaction;
the task policy continues and may still ask questions on its own. The rule
uses predicted gains relative to continuation under $\pi_T$, with no paired
rollouts or sampling correction at inference. We fix the threshold at zero
as implied by the RQV definition, without validation tuning.
To separate the policy's learned behavior from controller guidance, we also
evaluate the trained policy with semantic actions selected by $\pi_T^A$
instead of controller scores.

\begingroup
\setlength{\textfloatsep}{8pt plus 2pt minus 2pt}
\setlength{\floatsep}{8pt plus 2pt minus 2pt}
% Compact paragraph headings without negative glue leaking across page breaks.
\makeatletter
\renewcommand{\paragraph}{\@startsection{paragraph}{4}{\z@}%
  {1pt plus 0.5pt minus 0.5pt}{-1em}{\normalsize\bfseries}}
\makeatother
\section{Experiments}
\label{sec:experiments}

\vspace{-4pt}
\subsection{Experimental Setup}
\label{sec:exp_setup}

\begin{table}[!t]
\centering
\small
\setlength{\tabcolsep}{2.5pt}
\renewcommand{\arraystretch}{1.05}
\setlength{\abovecaptionskip}{0pt}
\setlength{\belowcaptionskip}{3pt}
\setlength{\parskip}{0pt}
\caption{Task performance with Qwen3-4B is reported as mean$\pm$std over
3 seeds (percentages except cumulative Intention credits),
\textbf{bold} and \underline{underlined} values indicate the best and second-best means, respectively.}
\label{tab:main}
\resizebox{\textwidth}{!}{%
\begin{tabular}{l cccc ccc c cc}
\toprule
\multirow{2}{*}{Method}
& \multicolumn{4}{c}{\textbf{UserGym}}
& \multicolumn{3}{c}{$\bm{\tau^2}$\textbf{-Bench}}
& \multicolumn{1}{c}{\textbf{IN3}}
& \multicolumn{2}{c}{\textbf{AskBench}}
\\
\cmidrule(lr){2-5}
\cmidrule(lr){6-8}
\cmidrule(lr){9-9}
\cmidrule(lr){10-11}
& Travel & Turtle & \textit{Inten.} & \textit{Tele.}
& Air. & Ret. & Telec.
& Judg.Acc
& AskMind & AskOv. \\
\midrule
\rowcolor{gray!15}
\multicolumn{11}{l}{\textit{Prompting-based}} \\

Vanilla
& \ms{26.74}{1.49}
& \ms{17.97}{1.73}
& \ms{185.50}{2.00}
& \ms{49.12}{1.87}
& \ms{23.78}{2.89}
& \ms{20.89}{2.61}
& \ms{20.00}{1.25}
& \ms{65.74}{1.54}
& \ms{38.70}{1.27}
& \ms{57.24}{1.18}
\\

ReAct
& \ms{28.66}{1.87}
& \ms{19.93}{1.63}
& \ms{192.17}{3.81}
& \ms{53.66}{2.44}
& \ms{26.67}{2.89}
& \ms{24.54}{1.29}
& \ms{20.83}{2.77}
& \ms{70.99}{1.42}
& \ms{44.34}{2.33}
& \ms{56.55}{1.18}
\\

Reflexion
& \ms{29.34}{1.76}
& \ms{21.70}{3.96}
& \ms{195.67}{3.95}
& \ms{52.03}{4.28}
& \ms{26.67}{1.55}
& \ms{24.26}{1.89}
& \ms{21.33}{2.20}
& \ms{71.61}{4.18}
& \ms{45.01}{1.59}
& \ms{58.75}{1.52}
\\

Proactive CoT
& \ms{27.01}{0.32}
& \ms{23.96}{4.06}
& \ms{176.17}{3.41}
& \ms{52.97}{2.84}
& \ms{25.44}{3.03}
& \ms{27.50}{2.50}
& \ms{20.83}{1.44}
& \ms{69.14}{2.33}
& \ms{42.58}{0.57}
& \ms{57.83}{1.23}
\\

PRIME
& \ms{28.20}{0.59}
& \ms{19.83}{2.58}
& \ms{173.17}{2.74}
& \ms{53.22}{4.45}
& \ms{25.00}{0.00}
& \ms{25.83}{1.44}
& \ms{18.33}{1.44}
& \ms{68.83}{1.42}
& \ms{41.24}{0.87}
& \ms{58.25}{1.00}
\\

\rowcolor{gray!15}
\multicolumn{11}{l}{\textit{Threshold-based}} \\

CLAM
& \ms{46.23}{2.88}
& \ms{10.56}{2.65}
& \underline{\ms{200.00}{3.95}}
& \ms{20.33}{1.41}
& \ms{28.53}{2.89}
& \ms{26.33}{2.20}
& \ms{21.67}{2.20}
& \ms{71.67}{1.07}
& \ms{45.01}{1.10}
& \ms{59.50}{1.25}
\\

INTENT-SIM
& \ms{27.59}{1.54}
& \ms{18.78}{1.26}
& \ms{198.42}{2.32}
& \ms{52.03}{1.41}
& \underline{\ms{33.33}{2.89}}
& \ms{27.17}{2.46}
& \ms{18.33}{2.89}
& \ms{73.06}{2.95}
& \ms{42.07}{2.06}
& \ms{59.00}{1.30}
\\

VoI-Gate
& \ms{49.19}{1.67}
& \ms{21.15}{2.45}
& \ms{198.08}{3.68}
& \ms{38.96}{2.63}
& \ms{31.67}{2.77}
& \ms{27.50}{3.01}
& \ms{21.17}{1.82}
& \ms{71.91}{1.42}
& \ms{44.64}{2.03}
& \ms{57.17}{1.88}
\\

\rowcolor{gray!15}
\multicolumn{11}{l}{\textit{Information-based training}} \\

IG-Reward
& \ms{43.75}{1.20}
& \ms{28.58}{2.99}
& \ms{194.68}{2.43}
& \ms{54.22}{3.45}
& \ms{25.00}{2.50}
& \ms{17.99}{3.82}
& \ms{22.06}{2.89}
& \ms{66.67}{4.24}
& \ms{42.87}{1.48}
& \ms{58.92}{1.76}
\\

InfoPO
& \ms{35.41}{1.13}
& \ms{25.52}{2.38}
& \ms{198.75}{3.88}
& \ms{53.66}{3.14}
& \ms{25.00}{0.00}
& \underline{\ms{30.02}{1.84}}
& \ms{22.50}{1.25}
& \ms{68.21}{5.35}
& \underline{\ms{46.32}{2.61}}
& \underline{\ms{60.15}{1.77}}
\\

IGPO
& \ms{31.62}{0.13}
& \ms{18.75}{0.63}
& \ms{194.25}{3.53}
& \ms{52.03}{3.84}
& \ms{26.67}{3.64}
& \ms{26.67}{2.89}
& \ms{16.67}{2.89}
& \ms{66.67}{2.83}
& \ms{45.31}{4.10}
& \ms{58.58}{0.29}
\\

\rowcolor{gray!15}
\multicolumn{11}{l}{\textit{Supervised and outcome-based training}} \\

SFT
& \underline{\ms{50.48}{1.76}}
& \ms{27.92}{3.11}
& \ms{198.25}{4.72}
& \ms{38.96}{3.63}
& \ms{25.77}{1.82}
& \ms{26.80}{3.12}
& \ms{21.83}{3.20}
& \underline{\ms{77.47}{5.26}}
& \ms{43.92}{1.19}
& \ms{58.50}{1.30}
\\

GRPO
& \ms{42.26}{0.94}
& \underline{\ms{28.96}{2.29}}
& \ms{196.42}{1.51}
& \underline{\ms{54.85}{3.08}}
& \ms{24.33}{3.23}
& \ms{27.50}{4.33}
& \underline{\ms{24.17}{2.89}}
& \ms{66.05}{3.25}
& \ms{39.65}{0.35}
& \ms{58.67}{0.63}
\\

\rowcolor{ourrow}
\textbf{\methodname (ours)}
& \textbf{\ms{63.72}{0.48}}
& \textbf{\ms{29.79}{1.15}}
& \textbf{\ms{209.17}{3.52}}
& \textbf{\ms{57.67}{2.45}}
& \textbf{\ms{35.00}{2.00}}
& \textbf{\ms{32.50}{1.29}}
& \textbf{\ms{24.94}{1.77}}
& \textbf{\ms{83.95}{1.42}}
& \textbf{\ms{48.37}{2.66}}
& \textbf{{\ms{61.27}{2.19}}}
\\
\bottomrule
\end{tabular}%
}
\par
\end{table}

\paragraph{Benchmarks.}
We evaluate on four benchmarks spanning ten settings.
\textbf{UserGym}~\citep{qian2025userrltraininginteractiveusercentric}
comprises Travel, Turtle, Intention, and Telepathy, with Intention and
Telepathy held out from training.
\textbf{$\tau^2$-Bench}~\citep{barres2025tau2benchevaluatingconversationalagents}
evaluates user--agent coordination and tool use in the Airline, Retail, and
Telecom domains. \textbf{IN3}~\citep{qian2024tellmoreimplicituser} evaluates
whether an agent recognizes underspecified instructions and recovers missing
task requirements. \textbf{AskBench}~\citep{zhao2026askaskbenchrubricguidedrlvr}
distinguishes requests with missing information in AskMind from requests
with false premises in AskOverconfidence.

\paragraph{Baselines.}
We compare four groups of baselines. \textbf{Prompting-based methods} use
fixed model weights and guide behavior through prompting or agent scaffolds,
including Vanilla~\citep{yang2025qwen3technicalreport},
ReAct~\citep{yao2023reactsynergizingreasoningacting},
Reflexion~\citep{shinn2023reflexionlanguageagentsverbal},
Proactive CoT~\citep{deng2023promptingevaluatinglargelanguage}, and
PRIME~\citep{wang2026primetrainingfreeproactive}.
\textbf{Threshold-based methods} include
CLAM~\citep{kuhn2023clamselectiveclarificationambiguous},
INTENT-SIM~\citep{zhang2023clarifynecessaryresolvingambiguity}, and VoI-Gate,
which score ambiguity, intent uncertainty, and expected clarification
utility, respectively. We tune their decision thresholds on validation data.
VoI-Gate is our threshold-based adaptation of value-of-information
reasoning~\citep{dong2026valueinformationframeworkhumanagent}.
\textbf{Information-based training methods} include
IG-Reward~\citep{deng2026uncertaintyawareclarificationllmagents},
InfoPO~\citep{kong2026infopoinformationdrivenpolicyoptimization}, and
IGPO~\citep{wang2026informationgainbasedpolicyoptimization}, which use
information-derived signals either alone or alongside task-return
supervision. \textbf{Supervised and outcome-based training methods} include
SFT, which learns from demonstrations, and
GRPO~\citep{shao2024deepseekmathpushinglimitsmathematical}, which constructs
group-relative advantages from task returns.

\paragraph{Evaluation and metrics.}
All main comparisons use Qwen3-4B as the backbone. We use the official
evaluation harnesses with our configured turn budgets and report main
results as mean and standard deviation over 3 seeds.
Table~\ref{tab:main} reports UserGym Pass-1 scores; domain-level task pass rates on
$\tau^2$-Bench; underspecification judgment accuracy (Judg.Acc) on IN3; and
answer accuracy (Acc) on AskBench. All columns are percentages except
UserGym Intention, whose per-detail credits accumulate within an episode.
Benchmark-specific settings, baseline implementations, and metric
definitions are detailed in Appendix~\ref{app:experiments}.

\vspace{-4pt}
\subsection{RQ1: Overall Performance}

\methodname achieves the highest mean score in all ten settings in
Table~\ref{tab:main}, including gains of $13.24$~points on Travel and
$6.48$~points in IN3 judgment accuracy over the strongest baselines.
Baseline strengths vary across settings. SFT leads the baselines on Travel
and IN3, while InfoPO leads on AskBench and GRPO on Turtle and Telepathy.
\methodname improves across these different interaction demands using the
same paired-outcome learning procedure. The gains on Intention and
Telepathy further support transfer to held-out environments without
additional training of either the policy or controller.

On AskBench and IN3, \methodname also achieves higher mean scores than
ACT~\citep{chen2025learningclarifymultiturnconversations} and
CollabLLM~\citep{wu2025collabllmpassiverespondersactive}, which learn from
interaction outcomes, across all six reported metrics
(Table~\ref{tab:outcome-baselines}). We next
examine how renewing value supervision and guiding collection contribute
to these gains, and how much the trained policy retains without controller
guidance at inference.

\Needspace{4\baselineskip}
\vspace{-4pt}
\subsection{RQ2: Where the Gain Comes From}

RQV compares targeted clarification with continuation under the current
policy. Better answer use can increase the benefit of asking, while
stronger native continuation can reduce it. This motivates renewing value
supervision as the policy learns. We test its practical benefit by freezing
the controller after bootstrap while continuing task-policy training.
Freezing lowers mean scores by $1.27$--$6.70$~points across all six settings
(Figure~\ref{fig:exploration}, bottom). Both variants retain controller-guided
collection, so this comparison assesses the benefit of updating the
controller within the training loop. The results support renewing
supervision over reusing guidance fitted to the initial backbone.
All variants retain policy learning from observed task returns, with no
additional reward from the controller's predictions.

\Needspace{22\baselineskip}
\begingroup
\setlength{\intextsep}{0pt}
\begin{wrapfigure}{r}{0.43\textwidth}
    \centering
    \setlength{\parskip}{0pt}
    \alignedwrapgraphic[trim=3.25bp 1bp 4bp 6bp,clip]{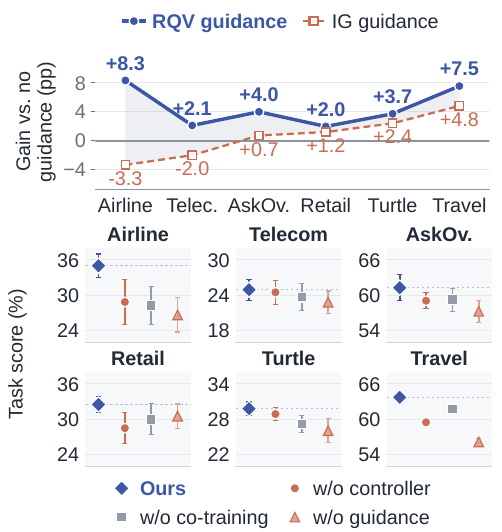}
    \par
    \caption{Effects of rollout guidance (top, gains over unguided
    collection in percentage points) and component ablations (bottom, mean$\pm$std over
    3 seeds) on task performance.}
    \label{fig:exploration}
\end{wrapfigure}

Removing value guidance during collection reduces mean scores by
$1.97$--$8.33$~points (Figure~\ref{fig:exploration}, bottom). This variant
retains paired comparisons, task-return advantages, sampling correction,
and controller updates, with five continuations per eligible history.
These gains support value-guided allocation for joint policy and controller
learning under a fixed rollout budget.

To test the guidance signal, we replace predicted RQV with an IG proxy in
the same proposal rule. RQV guidance achieves higher mean scores than both
unguided collection and IG guidance in all six settings
(Figure~\ref{fig:exploration}, top). Relative to unguided collection,
IG guidance improves four settings but reduces Airline and Telecom by
$3.34$ and $2.00$~points. With the proposal rule and task-return objective
retained, this comparison tests which training outcomes the guidance signal
favors. Under a shared turn budget, informative questions
may leave fewer turns for execution; this is one possible explanation for
the degradation on these two settings. RQV instead prioritizes the
estimated benefit relative to native continuation while keeping policy
updates grounded in observed returns.

The resulting gains also persist in the trained policy. Replacing the
controller with policy-based semantic action selection at inference
reduces mean scores by $0.44$--$6.14$~points, but the policy still exceeds
the strongest baseline on Travel ($59.49\%$ vs.\ $50.48\%$).
In five of six settings, the mean drop from removing collection guidance
exceeds that from removing controller-based selection at inference. This pattern
highlights the controller's role in shaping training experience alongside
its deployment role. Comparisons using the same constructor and realizer also support
gains beyond the menu interface (Appendix~\ref{app:ablations}).

\par
\WFclear
\endgroup
\Needspace{20\baselineskip}
\vspace{-4pt}
\subsection{RQ3: What the Agent Resolves}

\begingroup
\setlength{\intextsep}{0pt}
\begin{wrapfigure}{r}{0.43\textwidth}
    \centering
    \setlength{\parskip}{0pt}
    \alignedwrapgraphic[trim=9bp 7.74bp 4bp 4.25bp,clip]{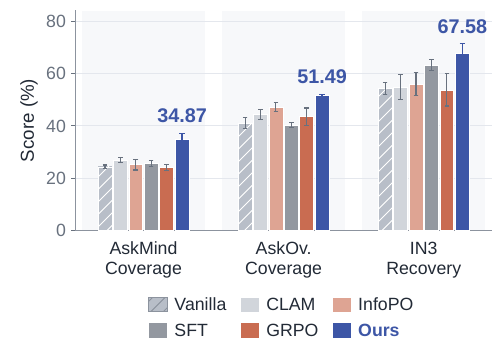}
    \par
    \caption{Instruction recovery on IN3 and complete clarification coverage
    on AskBench measure how fully agents resolve user requests
    (mean$\pm$std over 3 seeds).}
    \label{fig:clarification}
\end{wrapfigure}

Task success alone does not reveal how completely an agent resolves the
user's request. Figure~\ref{fig:clarification} examines this distinction.
On IN3, \emph{instruction recovery} (Recover) measures the
importance-weighted fraction of missing details resolved. On AskBench,
\emph{clarification coverage} (Cov) measures the fraction of tasks with \emph{every} required
point resolved, including missing information in AskMind and false premises
in AskOverconfidence. \methodname achieves the highest mean on all three
metrics, with recovery of $67.58\%$ and coverage of $34.87\%$ and $51.49\%$,
respectively. Because coverage requires resolving all required points,
these gains reflect more complete clarification, including correction of
false premises rather than only elicitation of missing details.

Yet coverage remains below answer accuracy on both AskBench settings.
An agent can therefore score well on final answers while leaving parts
of a request unresolved. IN3 separates another pair of capabilities.
Judgment accuracy evaluates whether the agent asks on underspecified
instructions and proceeds on fully specified ones, while instruction
recovery measures how much of the missing information the exchange
resolves. Improvements in both support gains
in recognizing clarification needs as well as recovering task-relevant
details. Identifying an underspecified request alone does not ensure that
subsequent questions obtain the information needed to complete it.
These metrics capture distinct aspects of interactive intent alignment,
from recognizing a need to clarify to resolving requirements and producing
a correct answer.

\par
\WFclear
\endgroup
\Needspace{19\baselineskip}
\vspace{-4pt}
\subsection{RQ4: Performance across Model Scales}
\label{sec:exp_scale}

\begingroup
\setlength{\intextsep}{0pt}
\begin{wrapfigure}{r}{0.4\textwidth}
    \centering
    \setlength{\parskip}{0pt}
    \alignedwrapgraphic[trim=7.75bp 0.75bp 3bp 3bp,clip]{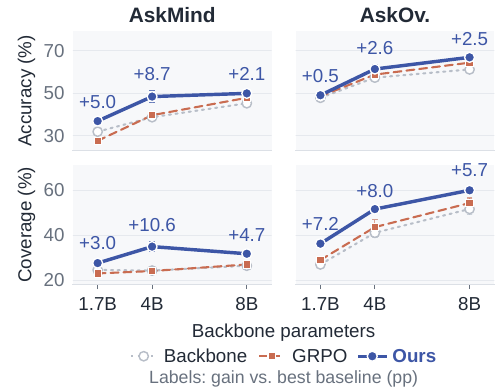}
    \par
    \caption{Answer accuracy and clarification coverage on AskBench,
    with gains over the stronger baseline in points.}
    \label{fig:scaling}
\end{wrapfigure}

At each model size (Qwen3-1.7B, 4B, and 8B), we compare \methodname with
the backbone and GRPO on AskBench, matching continuation budgets per
clarification-eligible history for the two training methods.
\methodname achieves higher mean accuracy
and coverage than both baselines
at all three sizes (Figure~\ref{fig:scaling}). At 8B, accuracy exceeds GRPO
by $2.08$~points on AskMind and $2.50$~points on AskOverconfidence.
Coverage improves by $4.74$ and $5.66$~points, respectively. The benefit of
clarification learning thus persists as the underlying policy becomes
stronger within the tested range.

Model scale does not improve both metrics in lockstep. For \methodname
on AskMind, mean answer accuracy increases with size, while mean coverage
peaks at 4B rather than 8B. These mean trends suggest that stronger
answering ability alone does not ensure more complete clarification.

\par
\WFclear
\endgroup

\endgroup

\Needspace{6\baselineskip}
\section{Related Work}
\label{sec:related_work}

\paragraph{Clarification and adaptive decisions.}
Research on clarification addresses how to resolve ambiguous information
needs~\citep{aliannejadi2020convai3,min2020ambigqa} in settings such as
question answering~\citep{kuhn2023clamselectiveclarificationambiguous,zhang2023clarifynecessaryresolvingambiguity},
code generation~\citep{mu2023clarifygptempoweringllmbasedcode}, and
planning~\citep{zhang2024askbeforeplanproactivelanguageagents}.
Interactive agent benchmarks extend this focus to implicit user
intentions and tool use~\citep{qian2024tellmoreimplicituser,appworldul2026},
incomplete instructions~\citep{underspecbench2026}, and communication
failures~\citep{driftbench2026,zhang2026oracle}.
These settings require agents to recognize missing information and decide
when to seek it. Value-of-information methods offer a decision-theoretic
view of this problem by weighing the expected benefit of obtaining
information against the cost of the
interaction~\citep{howard1966information,rao2018learningaskgoodquestions,dong2026valueinformationframeworkhumanagent}.

\paragraph{Learning from interaction outcomes.}
Another line of work evaluates questions through the interactions they
produce. Future-turn simulation~\citep{zhang2025modelingfutureconversationturns}
generates preference labels from anticipated dialogue, and
CollabLLM~\citep{wu2025collabllmpassiverespondersactive} incorporates subsequent
turns into its rewards.
STaR-GATE~\citep{andukuri2024stargateteachinglanguagemodels} uses gains in
personalized-response likelihood to select questions.
For dialogue-action learning,
ACT~\citep{chen2025learningclarifymultiturnconversations} constructs and
refreshes preference pairs through on-policy simulations.
Related ideas appear in retrieval, where
RICE-PO~\citep{li2026ricepoturningretrievalinteractions} assigns credit using
local counterfactual outcomes.

\paragraph{Information-based rewards and guided collection.}
A complementary approach measures how feedback changes an agent's beliefs
or subsequent decisions. IGPO rewards
answer-likelihood gains~\citep{wang2026informationgainbasedpolicyoptimization},
while clarification-specific rewards track user-goal belief
updates~\citep{deng2026uncertaintyawareclarificationllmagents}.
InfoPO~\citep{kong2026infopoinformationdrivenpolicyoptimization} combines
next-action likelihood changes under observed versus masked feedback with
outcome advantages.
Beyond reward design, experience
prioritization~\citep{schaul2016prioritizedexperiencereplay} addresses how to
allocate training effort across collected experiences.
\methodname connects these strands by using predicted clarification value
to guide collection and paired task outcomes to train both the value model
and the policy.
Appendix~\ref{app:related_work} discusses additional evaluation settings and
broader adaptive-agent mechanisms.

\section{Conclusion, Limitations and Future Work}
\label{sec:conclusion}

We presented \methodname for interactive intent alignment through
value-guided clarification. Paired interaction outcomes supervise both a
clarification-value controller and the task policy, while renewed
supervision tracks the policy's changing capabilities. Across four
benchmarks, \methodname improves task performance and clarification quality
with Qwen3-4B.

The fixed target constructor and question realizer restrict the controller's
available clarification choices. Training also relies on repeated
continuations from the same history and user goal, an assumption that is
difficult to satisfy in live interactions. Our evaluation uses simulated
users, so effectiveness with real users and evolving goals remains to be
established. Future work could jointly learn target construction and
question realization, and develop value supervision that requires fewer
paired rollouts.

% ICLR 2027 statements. All are placed after the main text, before the references,
% and do not count toward the page limit.

% \subsection*{The use of large language models (LLMs)}
% We used large language models to polish the writing and check grammar.
% The authors are responsible for the final manuscript.

% \subsection*{Ethics statement}
% Our experiments use existing benchmarks and simulated users.
% Real-world use should respect user privacy and avoid unnecessary clarification.

% \subsection*{Reproducibility statement}
% We provide implementation details, hyperparameters, and evaluation protocols
% in the appendix. Code is available at the repository linked in
% the abstract.



\bibliography{references}
\bibliographystyle{iclr2027_conference}

\clearpage
\appendix
\addtocontents{toc}{\protect\setcounter{tocdepth}{2}}
\begingroup
\renewcommand{\contentsname}{Appendix Contents}
\hypersetup{hidelinks,linktoc=all}
\pdfbookmark[0]{Appendix Contents}{appendix.contents}
\tableofcontents
\endgroup
\clearpage
\section{Additional Theory and Method Details}
\label{app:method}

\methodname learns clarification value relative to the task policy and uses
that value to guide the experiences from which the policy learns. This
appendix develops the reasoning behind this design. We first explain what
the value comparison should measure and how paired continuations provide
its supervision. We then derive the rollout allocation rule and the
sampling correction needed to learn from these selectively collected
outcomes. Finally, we connect the two learning processes through
co-training as the task policy evolves.

\subsection{Defining Policy-Relative Clarification Value}
\label{app:proofs}

To learn useful clarification, we need to distinguish the utility achieved
after an exchange from the utility added by that exchange. A successful
outcome alone cannot make this distinction because the agent may already
have been able to succeed by proceeding on its own. We therefore compare a
targeted clarification with native continuation from the same interaction
history under the same task policy.

We identify a clarification by its information target $z$, such as a travel
date or a budget, rather than by its surface wording. At history $h$, a
fixed constructor supplies candidate targets, and a fixed realizer turns a
selected target into a question. The comparison is between $\QUERY(z)$,
which obtains the corresponding user response before task execution
resumes, and $\PASS$, which delegates the continuation to the task policy,
including its native ability to ask. Let $Q^\pi(h,z)$ and $V^\pi(h)$ denote
the expected task utilities of these two continuations under policy $\pi$.
Their difference defines residual query value,
\[
g_\pi(h,z)=Q^\pi(h,z)-V^\pi(h).
\]
This reference makes clarification value specific to both the information
requested and the policy that will act on the answer. Both branches use the
same task utility, including any benchmark-specified interaction costs.

\paragraph{Information and decision value can rank questions differently.}
\label{app:proof-mismatch}

To see why this comparison differs from information gain, consider a
zero-cost, single-decision task with latent goal $G=(X,W)$, where
$X\sim\mathrm{Bernoulli}(1/2)$ and $W\sim\mathrm{Unif}\{1,\ldots,K\}$
are independent, with $K>2$. The agent receives reward one for predicting
$X$ correctly and zero otherwise. Revealing $W$ provides information gain
$\log K$ but leaves the best achievable expected reward at $1/2$, whereas
revealing $X$ provides only $\log 2$ information gain but raises that reward
to one. The more informative question therefore need not provide greater
task value, even when both answers are reliable and free to obtain.

Even when the answer is task-relevant, however, its benefit still depends
on the policy receiving it. Keeping the same task and question, consider
two policies that both predict zero without clarification. After observing
$X$, one follows the answer and achieves reward one, whereas the other
ignores it and retains expected reward $1/2$. Their residual query values
are therefore $1/2$ and zero, respectively, despite receiving exactly the
same information. This difference motivates evaluating clarification
through the outcomes of the current task policy rather than through the
information revealed alone.

\paragraph{Relation to classical value of information.}
\label{app:rqv-bound}

Unlike classical value of information, which evaluates the benefit to an
optimal decision maker, RQV measures the benefit under the policy that
actually continues the task. Because training can improve the $\QUERY$ and
$\PASS$ continuations by different amounts, the value of the same
clarification can change even when the information it reveals remains
fixed. We therefore estimate this policy-relative value from paired
outcomes generated by the policy being trained.

\subsection{Estimating RQV from Paired Continuations}
\label{app:interface}

To turn this policy-relative comparison into a learning signal, we collect
paired continuations that share the same history, latent user goal, and
task policy. Holding these quantities fixed isolates the initial
clarification decision from differences in the task or the agent. For each
selected target, the $\QUERY$ branch obtains a user response before
execution resumes, while the $\PASS$ branch delegates directly to the task
policy. Both then continue without further controller intervention,
including any questions the policy asks on its own.

Specifically, at a history $h$, the constructor supplies a finite candidate
set $\mathcal Z(h)$. When another user exchange is permitted and
$\mathcal Z(h)$ is nonempty, the semantic action set is
\[
\mathcal A(h)=\{\PASS\}\cup\{\QUERY(z):z\in\mathcal Z(h)\}.
\]
Otherwise, it reduces to $\{\PASS\}$.

\paragraph{Target construction and realization.}
The candidate constructor is fixed across training and inference. It uses
predefined information fields from the task interface where available;
otherwise, a frozen model constructs a small target set from the current
history. Target construction does not use hidden annotations, rollout
outcomes, task rewards, or RQV labels.
We hold the realized candidate set fixed within each comparison, so every
branch starts from the same decision menu.

Let $G\sim p(G\mid h)$ denote the latent user goal, which is held by the
environment rather than supplied to the policy or controller. Executing
$\QUERY(z)$ invokes a fixed realizer that produces a benchmark-native
question, followed by a user response $y\sim p_U(\cdot\mid G,h,z)$.
The exchange is appended to the history before task execution resumes.
Its tokens are environment-side and masked from the task-policy loss, so
the learnable decisions are which target to select and how to act on the
answer. Keeping the constructor and realizer fixed separates these two
sources of improvement from changes in candidate generation or question
wording.

\paragraph{Utility and interaction cost.}
\label{app:cost}
To compare the two branches on the same scale, we use the task utility
\begin{equation}
U(\tau,G)
=
R(\tau,G)
-
\sum_{z\in\mathcal{Q}(\tau)}c(z),
\label{eq:appendix-utility}
\end{equation}
where $R(\tau,G)$ is the benchmark return, $\mathcal Q(\tau)$ collects
clarification events with multiplicity, and $c(z)\ge0$ is the interaction
cost of an event. We set $c(z)=0$ when no such cost is specified. Comparing
\QUERY\ and \PASS\ under the same utility measures whether the information
obtained is worth the interaction required to obtain it.

\paragraph{From branch utilities to a paired estimate.}
With this common utility, we can estimate the comparison by sampling both
continuations under a fixed collecting policy $\pi$. Let $\tau_a$ denote
a continuation whose initial semantic action is fixed to $a$ and whose
subsequent behavior follows $\pi$. Conditional on the same history and
user goal, the two branch values are
\begin{equation}
\begin{aligned}
Q^\pi(h,z;G)&=\E[U(\tau_{\QUERY(z)},G)\mid h,G],\\
V^\pi(h;G)&=\E[U(\tau_{\PASS},G)\mid h,G].
\end{aligned}
\label{eq:goal-conditional-values}
\end{equation}
These expectations include the user response and all subsequent rollout
randomness. Each group holds $h$ and $G$ fixed, and its arms and replica
counts are chosen before observing continuation outcomes. We then sample
continuations independently conditional on these quantities. If a selected
target receives $K_z$ query continuations with utilities $U_{z,k}$ and
the shared reference contains $K_p$ continuations with utilities
$U_{\PASS,j}$, Equation~\ref{eq:rqv-estimator} becomes
\[
\widehat g_\pi(h,z)
=\bar U_{\QUERY(z)}-\bar U_{\PASS}
=\frac{1}{K_z}\sum_{k=1}^{K_z}U_{z,k}
-\frac{1}{K_p}\sum_{j=1}^{K_p}U_{\PASS,j}.
\]
Conditional on the allocated arms and their replica counts, each sample
mean has expectation equal to its branch value. Averaging over the
allocation therefore preserves these expectations, and linearity gives
\[
\begin{aligned}
\E[\widehat g_\pi(h,z)\mid h,G,\ z\text{ included}]
&=\E[\bar U_{\QUERY(z)}\mid h,G,\ z\text{ included}]\\
&\quad-\E[\bar U_{\PASS}\mid h,G,\ z\text{ included}]\\
&=Q^\pi(h,z;G)-V^\pi(h;G).
\end{aligned}
\]
Thus matching the user goal makes the paired difference estimate the
effect of the initial clarification, rather than a difference in what the
two users wanted. To recover the history-conditioned RQV in
Section~\ref{sec:setup}, we next average over the latent goal. In
\methodname, target construction and selection use the visible history, not the hidden goal
or future outcomes, so including $z$ does not change $p(G\mid h)$.
The law of total expectation then yields
\begin{equation}
\begin{aligned}
\E[\widehat g_\pi(h,z)\mid h,\ z\text{ included}]
&=\E_{G\mid h}\bigl[Q^\pi(h,z;G)-V^\pi(h;G)\bigr]\\
&=g_\pi(h,z).
\end{aligned}
\label{eq:paired-unbiased}
\end{equation}

\paragraph{Learning from the paired labels.}
Accordingly, we fit $f_\phi(h,z)$ to these paired differences using the
squared-error objective in Equation~\ref{eq:controller-loss}.
Squared-error regression targets the conditional mean of the label, which
the derivation above identifies as RQV under the collecting policy. The
controller thereby learns to predict the utility added by clarification,
rather than the utility attained after asking alone, and scores new
state--target pairs without observing the latent goal or collecting paired
continuations at inference.

\subsection{Value-Guided Collection under a Fixed Budget}
\label{app:rollouts}

Paired evaluation identifies the value of a selected target, but evaluating
every candidate at every history would spread the rollout budget thinly.
We instead combine value-guided target selection with a separate coverage
draw and repeated continuations. These serve complementary purposes.
Guidance prioritizes promising questions, coverage collects evidence beyond
that prioritization, and replication improves the precision of the resulting
comparisons.

\paragraph{Semantic-action probabilities.}
Before applying value guidance, we construct a reference distribution that
represents the task policy's own preferences. The task policy defines a
categorical policy over the menu $\mathcal A(h)$ without an additional
action head. Each action has a nonempty canonical serialization $s(a)$,
scored by its mean token log-probability conditioned on the history and
available actions,
\[
L_\theta(a\mid h)
=
\frac{1}{|s(a)|}\log\pi_\theta\bigl(s(a)\mid h,\,\mathcal{A}(h)\bigr),
\]
where $|s(a)|$ is its token count. Normalizing these length-adjusted scores
gives
\begin{equation}
\pi_\theta^A(a\mid h)
=
\frac{\exp L_\theta(a\mid h)}
{\sum_{a'\in\mathcal{A}(h)}\exp L_\theta(a'\mid h)}.
\label{eq:semantic-action-prob}
\end{equation}
This distribution provides both the policy reference for collection and
the on-policy arm weights used in Proposition~\ref{prop:selection}.

\paragraph{A policy-based reference with full support.}
Conditioned on asking, the task policy induces
\[
p_\pi(z\mid h)
=
\frac{\pi_\theta^A(\QUERY(z)\mid h)}
{\sum_{z'}\pi_\theta^A(\QUERY(z')\mid h)}.
\]
We use this distribution as the starting point for exploration, retaining
the policy's preferences among targets. To keep low-probability candidates
reachable, we mix it with a uniform distribution,
\[
\bar p_\pi(z\mid h)
=(1-\epsilon)p_\pi(z\mid h)
+\frac{\epsilon}{|\mathcal Z(h)|},
\qquad 0<\epsilon<1.
\]
This gives every candidate positive reference probability before applying
value guidance.

\paragraph{Guided proposal.}
\label{app:proposal-proof}

With this reference in place, we use the controller to concentrate part of
the rollout budget on clarifications with higher predicted value. Rather
than replace the policy's preferences with a greedy choice, we penalize
departure from the reference. At fixed $h$ and fixed collecting policy and
controller, write $p(z)=\bar p_\pi(z\mid h)>0$ and assume finite predicted
values. For $\beta>0$, the resulting tradeoff is
\[
\max_{\mu\in\Delta(\mathcal Z(h))}
\E_{z\sim\mu}[f_\phi(h,z)]-\beta^{-1}\KL(\mu\Vert p).
\]
To derive the proposal, introduce a multiplier $\lambda$ for the
normalization constraint $\sum_z\mu(z)=1$. Differentiating with respect to
each positive $\mu(z)$ gives the stationarity condition and its solution,
\[
\begin{aligned}
0&=f_\phi(h,z)-\frac{1}{\beta}
\left(\log\frac{\mu(z)}{p(z)}+1\right)+\lambda,\\
\mu(z)&=p(z)\exp\{\beta f_\phi(h,z)+\beta\lambda-1\}.
\end{aligned}
\]
The multiplier contributes the same factor to every target. Enforcing
normalization therefore yields
\[
\mu^*(z\mid h)
=
\frac{p(z)\exp\{\beta f_\phi(h,z)\}}
{\sum_{z'}p(z')\exp\{\beta f_\phi(h,z')\}},
\]
which recovers Equation~\ref{eq:guided-proposal}. Positive reference
probabilities and finite predictions make this solution interior, and
strict concavity of the objective makes it the unique maximizer. As
$\beta\to0$, the proposal approaches $p$; increasing $\beta$ gives
predicted value more influence relative to the reference. The tilt
prioritizes predicted RQV, while its effect on actual task utility depends
on how these predictions align with current-policy outcomes.

\paragraph{Coverage complements guidance.}
To obtain evidence beyond this prioritization, we reserve a separate draw
for reference-based sampling.
When at least two candidates are available, we first draw
$z_g\sim\mu^*(\cdot\mid h)$, then draw a distinct target from the
reference distribution over the remaining candidates,
\[
\Pr(z_c=z\mid h,z_g)
=\frac{\bar p_\pi(z\mid h)}{1-\bar p_\pi(z_g\mid h)},
\qquad z\ne z_g.
\]
Support smoothing and this second draw address different needs. Smoothing
keeps every target reachable, whereas the coverage draw reserves rollout
budget for sampling without the value tilt among the remaining targets.
Excluding $z_g$ ensures that this budget evaluates an additional candidate
rather than replicating the guided choice. Guidance uses the controller's
current estimates; coverage supplies comparisons that can revise them.

\paragraph{A shared reference for paired comparisons.}
Having chosen which targets to evaluate, we allocate continuations to the
query arms and a shared $\PASS$ reference. Reusing the same $\PASS$ sample
mean supports both comparisons without collecting a separate reference for
each target. With at least two candidates, the group contains $K_p\ge2$
$\PASS$ continuations, $K_g\ge1$ continuations for the guided target, and
$K_c\ge1$ for the coverage target,
\[
S_h
=
\{
\PASS^{\times K_p},\;
\QUERY(z_g)^{\times K_g},\;
\QUERY(z_c)^{\times K_c}
\}.
\]
Using at least two $\PASS$ replicas also allows each to use the others as
its policy-learning baseline, as developed in the next subsection. When
$|\mathcal Z(h)|=1$, the single query target receives $K_g+K_c$
continuations. The GRPO comparison uses the same total continuation budget
per clarification-eligible history. Experimental allocations are given in
Appendix~\ref{app:implementation}.

\paragraph{Precision under a fixed budget.}
To quantify the effect of replication, fix the collecting policy and
condition on $\mathcal C$, comprising the shared history $h$, latent goal
$G$, selected arms, and replica counts. Let $\sigma_z^2$ and
$\sigma_P^2$ denote the finite per-continuation return variances under
$\QUERY(z)$ and $\PASS$, respectively. Since continuations are sampled
independently within and across arms, the variance of their sample-mean
difference is
\begin{equation}
\begin{aligned}
\Var\bigl(\widehat g_\pi(h,z)\mid\mathcal C\bigr)
&=\Var\bigl(\bar U_{\QUERY(z)}\mid\mathcal C\bigr)
+\Var\bigl(\bar U_{\PASS}\mid\mathcal C\bigr)\\
&=\frac{\sigma_z^2}{K_z}+\frac{\sigma_P^2}{K_p},
\end{aligned}
\label{eq:paired-variance}
\end{equation}
where $K_z$ is the replica count of the selected query arm.

The two terms explain how replication improves the comparison. Increasing
$K_p$ reduces reference noise in both labels, whereas increasing $K_z$
reduces the query-specific noise for target $z$. The two labels share
reference noise, which cancels when comparing their empirical values with
each other but still affects either comparison against proceeding. Under
the fixed total budget, we
assign more query replicas to the guided target, sharpening the prioritized
comparison while retaining a coverage arm to evaluate another candidate.

\Needspace{5\baselineskip}
\subsection{Policy Learning with Sampling Correction}
\label{app:selection-proof}

Guided collection changes which arms are observed and how many
continuations each receives, but not the task utility that the policy is
trained to improve. We therefore start from the task-return objective under
the policy's own action distribution, then derive how the shared $\PASS$
reference and the sampling correction recover its gradient from the
structured rollout groups.

\paragraph{The task-return objective.}
Fix the history $h$, candidate set, and conditional goal distribution
$p(G\mid h)$. Let $P_\theta^a(\tau\mid h,G)$ denote the law of a
continuation whose initial action is fixed to $a$ and whose subsequent
behavior follows the task policy. Conditional on $G$, the expected utility
under the policy's own arm selection is
\[
J_h(\theta;G)
=\sum_{a\in\mathcal A(h)}\pi_\theta^A(a\mid h)
\E_{\tau\sim P_\theta^a}[U(\tau,G)].
\]
Averaging over the latent goal gives
$J_h(\theta)=\E_{G\mid h}[J_h(\theta;G)]$. We first derive the update
conditional on $G$ and take this final average below. All gradients are
evaluated at the collecting policy. We assume differentiable continuation
laws with locally parameter-independent support and that differentiation
can be interchanged with expectation; utility and environment kernels have
no direct dependence on $\theta$.

Both arm selection and subsequent execution depend on the task-policy
parameters. Their joint score is
\begin{equation}
S_\theta(a,\tau\mid h,G)
=\nabla_\theta\log\pi_\theta^A(a\mid h)
+\nabla_\theta\log P_\theta^a(\tau\mid h,G).
\label{eq:continuation-score}
\end{equation}
The second term sums token scores over the agent's native continuation
after the initial semantic action. Fixed question realization, user
responses, and other environment transitions contribute no score, and
$G$ indexes the environment law rather than becoming an input to the
policy. Applying the product rule and score identity to $J_h$ therefore
gives the gradient we aim to recover,
\[
\nabla_\theta J_h(\theta;G)
=\sum_a\pi_\theta^A(a\mid h)
\E_{\tau\sim P_\theta^a}[U(\tau,G)S_\theta(a,\tau\mid h,G)].
\]

\paragraph{PASS-referenced advantages.}
The paired group already supplies a common task-return reference.
Reusing it for policy learning makes each query advantage measure utility
relative to proceeding, consistently with the controller's supervision.
For a query continuation, we use
\[
A_{z,k}=U_{z,k}-\bar U_{\PASS}.
\]
For a $\PASS$ continuation, however, including its own return in the
reference would make the baseline depend on the execution it evaluates.
We therefore leave that continuation out,
\[
A_{\PASS,k}
=
U_{\PASS,k}
-
\frac{1}{K_p-1}
\sum_{j\neq k}U_{\PASS,j}.
\]
The mean query advantage is exactly the empirical RQV label in
Equation~\ref{eq:rqv-estimator}, so the policy and controller learn from
the same outcome comparison at trajectory and arm level, respectively.
We treat advantages as fixed coefficients in the policy update and do not
divide them by a group-dependent standard deviation.

\paragraph{Why the reference preserves the objective.}
Write $B_{a,k}$ for the reference subtracted from $U_{a,k}$. Conditional
on $h,G$ and the selected arms, $B_{a,k}$ is independent of its own
continuation and has expectation $V^\pi(h;G)$ for every arm. This follows
from independent continuation sampling, the fixed number of $\PASS$
replicas, and leave-one-out construction for the $\PASS$ arm itself.
Because each continuation law is normalized, its expected score is zero,
\[
\E_{\tau\sim P_\theta^a}
[\nabla_\theta\log P_\theta^a(\tau\mid h,G)]=0.
\]
The independent baseline therefore contributes only the arm-selection
score in expectation,
\[
\E[B_{a,k}S_\theta(a,\tau_{a,k}\mid h,G)\mid h,G,\ a\text{ included}]
=V^\pi(h;G)\nabla_\theta\log\pi_\theta^A(a\mid h).
\]
Crucially, the expected reference is the same for all arms, so weighting
by the policy's arm probabilities makes the remaining term cancel,
\[
V^\pi(h;G)\sum_a\pi_\theta^A(a\mid h)
\nabla_\theta\log\pi_\theta^A(a\mid h)
=V^\pi(h;G)\nabla_\theta\sum_a\pi_\theta^A(a\mid h)=0.
\]
Thus replacing returns with $\PASS$-referenced advantages preserves the
gradient of $J_h(\theta;G)$ for both arm selection and execution. Although
a $\PASS$ continuation has zero expected advantage, its advantage can
still covary with the continuation score, so its execution gradient need
not vanish. Native task behavior can therefore improve on $\PASS$
trajectories as well as on query trajectories.

\paragraph{Averaging within each arm.}
We now account for the two ways guided collection changes the data.
First, an arm can receive more replicas without being intended to carry
more weight in the task objective. We therefore average its contributions,
\[
\bar F_a=\frac{1}{K_a}\sum_{k=1}^{K_a}
A_{a,k}S_\theta(a,\tau_{a,k}\mid h,G).
\]
Let $m_a(h,G)$ denote the expected contribution of one on-policy
continuation of arm $a$ with the same reference construction. Conditional
on inclusion, the continuation law and expected reference are unchanged,
whether the target was selected for guidance or coverage. Even when these
roles receive different replica counts, we therefore have
\[
\E[\bar F_a\mid h,G,\ a\text{ included}]=m_a(h,G).
\]
The baseline calculation above establishes
$\sum_a\pi_\theta^A(a\mid h)m_a(h,G)=\nabla_\theta J_h(\theta;G)$.

\paragraph{Correcting arm inclusion.}
Second, guided collection makes some arms more likely to appear in a
group. Let $\mathcal I_h$ be the set of included arms and
$\rho_h(a)=\Pr(a\in\mathcal I_h\mid h)$ their inclusion probabilities.
For $|\mathcal Z(h)|\ge2$, write $\mu(z)=\mu^*(z\mid h)$ and
$p(z)=\bar p_\pi(z\mid h)$. A query arm enters either as the guided
target or as the coverage target after a different guided draw. Summing
these disjoint events gives
\begin{equation}
\rho_h(\QUERY(z))
=\mu(z)+\sum_{z'\ne z}\mu(z')\frac{p(z)}{1-p(z')}.
\label{eq:inclusion-prob}
\end{equation}
The $\PASS$ arm is always included, so $\rho_h(\PASS)=1$. When there is
only one candidate, its query arm also has inclusion probability one.

Having removed replica-count weighting, we correct this second effect
by weighting each included arm mean by $\pi_\theta^A(a\mid h)/\rho_h(a)$,
as in Equation~\ref{eq:ht-estimator}. Write
$I_a=\mathbf 1[a\in\mathcal I_h]$, with zero contribution for excluded
arms. Target selection is independent of $G$ conditional on $h$, so
$\E[I_a\mid h,G]=\rho_h(a)$. Taking expectations makes the correction
explicit,
\begin{equation}
\begin{aligned}
\E[\widehat H(h)\mid h,G]
&=
\sum_a\frac{\pi_\theta^A(a\mid h)}{\rho_h(a)}
\E[I_a\bar F_a\mid h,G]\\
&=\sum_a\frac{\pi_\theta^A(a\mid h)}{\rho_h(a)}
\rho_h(a)m_a(h,G)\\
&=\sum_a\pi_\theta^A(a\mid h)m_a(h,G)
=\nabla_\theta J_h(\theta;G).
\end{aligned}
\label{eq:conditional-gradient-preservation}
\end{equation}
Finally, averaging over the fixed distribution $G\mid h$ gives
\[
\E[\widehat H(h)\mid h]
=\sum_a\pi_\theta^A(a\mid h)m_a(h)
=\nabla_\theta J_h(\theta),
\]
where $m_a(h)=\E_{G\mid h}[m_a(h,G)]$ and
the objective $J_h(\theta)$ is the one defined at the start of this
subsection. This proves Proposition~\ref{prop:selection} for the
fixed-history task-return contribution of both action selection and
execution.

All inclusion weights are treated as fixed when forming this estimator.
The controller can therefore change which arms receive data without giving
those arms extra weight in the expected task-return gradient at $h$.
Reference-policy regularization and the subsequent optimization step are
specified separately in Appendix~\ref{app:implementation}.

\paragraph{Coverage bounds the correction weights.}
The coverage draw also controls the size of the correction, independently
of how strongly the controller concentrates the guided proposal.

\Needspace{12\baselineskip}
\begin{proposition}[Coverage-Induced Bounds on Correction Weights]
\label{prop:coverage-weight-bound}
Fix a clarification-eligible history $h$ and $0<\epsilon<1$. Use the
smoothed reference $p=\bar p_\pi$ and the guided-plus-coverage selection
rule in Equation~\ref{eq:inclusion-prob}, including the single-candidate
convention. For any guided distribution $\mu$ and every candidate $z$,
\begin{equation}
\begin{aligned}
\frac{\pi_\theta^A(\QUERY(z)\mid h)}{\rho_h(\QUERY(z))}
&\le\frac{1}{1-\epsilon},\\
\frac{\pi_\theta^A(\PASS\mid h)}{\rho_h(\PASS)}
&=\pi_\theta^A(\PASS\mid h)\le1.
\end{aligned}
\label{eq:selection-weight-bound}
\end{equation}
\end{proposition}

\begin{proof}
For at least two candidates, $0<p(z)<1$, and
Equation~\ref{eq:inclusion-prob} implies
\[
\rho_h(\QUERY(z))
\ge\mu(z)+p(z)\sum_{z'\ne z}\mu(z')
=p(z)+(1-p(z))\mu(z)
\ge p(z).
\]
Since $p_\pi(z\mid h)$ conditions on asking,
$p_\pi(z\mid h)\ge\pi_\theta^A(\QUERY(z)\mid h)$. Support smoothing
then gives
\[
\rho_h(\QUERY(z))\ge p(z)
\ge(1-\epsilon)p_\pi(z\mid h)
\ge(1-\epsilon)\pi_\theta^A(\QUERY(z)\mid h).
\]
Dividing by the positive inclusion probability gives the query bound.
The single-candidate case satisfies the same bound because its inclusion
probability is one. The $\PASS$ bound follows from
$\rho_h(\PASS)=1$.
\end{proof}

The bound depends neither on the guidance strength nor
on the range of controller predictions. Thus, coverage serves both to
broaden supervision and to keep guided collection compatible with bounded,
untruncated correction weights. With our $\epsilon=0.2$, every inclusion
weight is at most $1.25$.

\paragraph{Relation to information-based training signals.}
\label{app:ig-signals}

Information-based objectives provide intermediate supervision by measuring
how an observation changes a reference quantity, such as the likelihood of
a gold answer, the inferred user goal, or subsequent agent behavior
\citep{wang2026informationgainbasedpolicyoptimization,
deng2026uncertaintyawareclarificationllmagents,
kong2026infopoinformationdrivenpolicyoptimization}.
These signals can also be combined with task-return supervision. In
\methodname, predicted RQV instead guides which experiences are collected,
while policy advantages use observed task utilities and the correction
above preserves their intended arm weighting. The controller's prediction
thus allocates the learning opportunity rather than becoming an additional
reward, allowing the controller and task policy to learn different
quantities from the same outcomes.

\subsection{Online Co-Training and Inference}
\label{app:online}

Co-training follows from the policy-relative definition of RQV. As the
task policy learns, it can improve both its use of an answer and its native
continuation, changing the comparison that the controller should predict.
For a fixed history $h$ and target $z$, subtracting
Equation~\ref{eq:rqv} across consecutive policy versions gives
\[
\begin{aligned}
g_{\pi_{t+1}}(h,z)-g_{\pi_t}(h,z)
&=\bigl[Q^{\pi_{t+1}}(h,z)-Q^{\pi_t}(h,z)\bigr]\\
&\quad-\bigl[V^{\pi_{t+1}}(h)-V^{\pi_t}(h)\bigr].
\end{aligned}
\]
Improved use of an answer can increase its residual value, whereas improved
native execution can make the same exchange redundant. What matters is
the relative change in the two branches, which motivates renewing the
controller's supervision as the policy evolves rather than treating the
values learned at bootstrap as fixed.

\paragraph{One collection phase, two updates.}
We implement this feedback loop through a shared rollout batch. First, we
bootstrap $f_0$ on paired outcomes under $\pi_0$, drawing targets from the
policy-based reference without value guidance. At step $t$, we hold
$\pi_t$ and $f_t$ fixed during collection, so the controller guides target
selection while all continuations follow the same task policy. The batch
then provides paired RQV labels for fitting $f_{t+1}$ and sampling-corrected
task-return gradients for updating $\pi_t$ to $\pi_{t+1}$. The controller
learns the value of intervention, while the policy learns both target
selection and subsequent task execution from observed returns. Controller
fitting uses only the current batch, keeping its labels tied to one
collecting policy. The next round uses both updated components, so the
controller shapes exploration and the policy's resulting behavior supplies
fresh value supervision.

\paragraph{One-update lag and final refresh.}
\label{app:staleness}
Because both updates use outcomes generated before the policy step,
$f_{t+1}$ is fitted to $g_{\pi_t}$ and then guides collection under
$\pi_{t+1}$. This introduces a one-update lag between value supervision
and the policy being guided. The next batch renews supervision under
$\pi_{t+1}$, tying the targets to successive policy versions while retaining
only one collection phase per training step.

After the final policy update, we hold $\pi_T$ fixed and collect fresh
paired outcomes solely to refresh the controller, without another policy
update. Its regression targets are therefore generated by the same policy
used at inference. Algorithm~\ref{alg:rqpo} summarizes bootstrap, joint
updates, and final refresh, using the group allocation in
Appendix~\ref{app:rollouts}.

\begin{algorithm}[!ht]
\caption{Online co-training in \methodname}
\label{alg:rqpo}
\setlength{\tabcolsep}{3.5pt}
\renewcommand{\arraystretch}{1.15}
\begin{tabularx}{\linewidth}{@{}>{\color{gray}}r X@{}}
\multicolumn{2}{@{}l@{}}{\textcolor{algorithmink}{\textbf{Require:}} task policy $\pi_0$, updates $T$, budgets $(K_p,K_g,K_c)$}\\
1: & \textcolor{algorithmink}{\textbf{Bootstrap}} $f_0$ on paired rollouts collected with $\mu=\bar p_{\pi_0}$\\[2pt]
2: & \textcolor{algorithmink}{\textbf{for}} $t=0,1,\ldots,T-1$ \textcolor{algorithmink}{\textbf{do}}\\
3: & \quad Sample a batch $\mathcal B_t$ of clarification-eligible states $(h,G)$\\
4: & \quad\textcolor{algorithmink}{\textbf{for each}} $(h,G)\in\mathcal B_t$ \textcolor{algorithmink}{\textbf{do}}\\
5: & \qquad$\mu_t(z\mid h)\propto\bar p_{\pi_t}(z\mid h)\exp\{\beta f_t(h,z)\}$\\
6: & \qquad Draw guided target $z_g\sim\mu_t(\cdot\mid h)$\\
7: & \qquad Draw a distinct coverage target from renormalized $\bar p_{\pi_t}$ when available\\
8: & \qquad Collect paired continuations $S_h$ under $\pi_t$ at the same $(h,G)$\\
9: & \qquad Compute $\widehat g_t(h,z)=\bar U_{\QUERY(z)}-\bar U_{\PASS}$ for each selected target\\
10: & \qquad Form $\PASS$-referenced contributions and compute $\widehat H(h)$ using Eq.~\ref{eq:ht-estimator}\\
11: & \quad\textcolor{algorithmink}{\textbf{end for}}\\[2pt]
\rowcolor{algorithmwash}
12: & \quad\textcolor{algorithmink}{\textbf{Fit}} $f_{t+1}$ to the batch's paired value targets using Eq.~\ref{eq:controller-loss}\\
\rowcolor{algorithmwash}
13: & \quad$\pi_{t+1}\leftarrow\textcolor{algorithmink}{\textsc{Update}}\bigl(\pi_t,\,|\mathcal B_t|^{-1}\sum_{(h,G)\in\mathcal B_t}\widehat H(h)\bigr)$\\
14: & \textcolor{algorithmink}{\textbf{end for}}\\[2pt]
15: & \textcolor{algorithmink}{\textbf{Refresh}} $f$ on fresh paired rollouts under the fixed final policy $\pi_T$\\
16: & \textcolor{algorithmink}{\textbf{return}} task policy $\pi_T$ and refreshed controller $f$\\
\end{tabularx}
\end{algorithm}

\paragraph{From exploration to intervention.}
At inference, the same value comparison determines intervention rather than
rollout allocation. While active, the refreshed controller follows
Equation~\ref{eq:inference-policy}, selecting the highest-valued query only
when its predicted RQV is positive and otherwise choosing $\PASS$. Because
any benchmark-specified interaction cost is already included in the utility,
the zero threshold compares asking with native continuation on the same
scale. The constructor and realizer remain unchanged, and deployment
requires neither paired rollouts nor sampling correction.

Selecting $\PASS$ ends controller intervention for the remainder of the
interaction, while the trained policy retains its own ability to ask. The
\emph{w/o controller} variant retains the trained policy, constructor,
and realizer, but selects menu actions using $\pi^A$ instead of controller
scores.

\section{Implementation Details}
\label{app:implementation}

Across benchmarks, \methodname follows a shared paired-outcome learning
framework with shared training settings and benchmark-specific interaction
interfaces. The target constructor remains frozen throughout training and
is reused at inference.

\paragraph{Data and training schedule.}
We follow the data partitions used by
InfoPO~\citep{kong2026infopoinformationdrivenpolicyoptimization} and the
original benchmarks. The fixed training pool is drawn from the training
split. A development set of $64$ tasks from the tuning split is used only
for monitoring, and the official test split is reserved for final
evaluation. Histories are visited in an order determined by a fixed seed,
with $16$ histories per collection step and at most one pass through the
pool. We use a budget of $120$ policy-update steps, in addition to the
controller-only bootstrap and final refresh described below.

\paragraph{Controller parameterization.}
The controller $f_\phi$ predicts RQV from the interaction history and a
candidate information target. Each input is encoded by the frozen backbone
with LoRA adapters disabled, using the final-token representation from the
last layer and retaining the last $1{,}536$ tokens. The target is serialized
as \texttt{Information target:} followed by $z$.
Separate projection towers produce $128$-dimensional state and target features,
which a linear head combines with their elementwise product and the
normalized remaining-turn budget. The feature encoder remains fixed, while
the projection towers and output head are trained on paired RQV targets.

\begin{table}[!t]
\centering
\caption{Shared training configuration of \methodname.}
\label{tab:implementation_config}
\small
\setlength{\tabcolsep}{5pt}
\renewcommand{\arraystretch}{1.12}
\begin{tabularx}{\linewidth}{@{}>{\raggedright\arraybackslash}p{0.18\linewidth}>{\raggedright\arraybackslash}p{0.43\linewidth}>{\raggedright\arraybackslash}X@{}}
\toprule
\textbf{Component} & \textbf{Parameter} & \textbf{Setting} \\
\midrule
\textbf{Task policy}
& LoRA rank / scaling $\alpha$ & $16$ / $32$ \\
& Adapted modules & Attention and MLP projections \\
& Optimizer / learning rate & AdamW / $2\times10^{-5}$ \\
& Policy-ratio clipping & $0.2$ \\
& Reference-KL coefficient & $0.01$ \\
& Maximum context length & $8{,}192$ tokens \\
\addlinespace[4pt]
\textbf{Controller}
& Maximum history length & $1{,}536$ tokens \\
& Projection block & LayerNorm $\to$ Linear $\to$ $\tanh$ \\
& Projection dimension & $128$ \\
& Output head & Linear ($385\to1$) \\
& Remaining-turn feature & $b/16$ \\
& Output-layer initialization & Zero \\
& Trainable parameters (4B backbone) & Approximately $0.67$M \\
& Optimizer / learning rate & AdamW / $3\times10^{-4}$ \\
& Fitting passes per collection batch & $4$ \\
\addlinespace[4pt]
\textbf{Both optimizers}
& Maximum gradient norm & $1.0$ \\
\addlinespace[4pt]
\textbf{Training rollouts}
& Histories per collection step & $16$ \\
& Maximum candidate targets $m$ & $6$ \\
& Continuations $(K_p,K_g,K_c)$ & $(2,2,1)$ \\
& Support smoothing $\epsilon$ & $0.2$ \\
& Guidance strength $\beta$ & $1.0$ \\
& Temperature / top-$p$ & $1.0$ / $1.0$ \\
& Maximum generated tokens per turn & $1{,}024$ \\
\bottomrule
\end{tabularx}
\par\smallskip
\begin{minipage}{\linewidth}
\footnotesize
The rollout allocation applies to states with at least two candidate targets;
with one candidate, its query arm receives the combined query budget.
Here $b$ denotes the number of turns remaining in the episode.
\par
For AskBench and IN3, task utility sums benchmark evaluator scores for
requirement fulfillment and answer correctness; UserGym and $\tau^2$-Bench
use environment-provided task rewards.
\end{minipage}
\end{table}

\paragraph{Rollout collection.}
At each selected clarification-eligible history, we generate fresh paired
$\PASS$ and $\QUERY$ continuations under the current task policy, using the allocation in
Table~\ref{tab:implementation_config}. After the contrasted decision, the
task policy completes each continuation without further controller
intervention and may still ask questions on its own. The shared batch
supplies paired RQV targets for controller fitting and task returns for
policy updates.

\paragraph{Bootstrap and final refresh.}
Before policy training, we bootstrap the controller over four collection
steps with value guidance disabled ($\beta=0$), using $64$ histories and
$320$ continuations while holding the task policy fixed. Controller fitting
minimizes unweighted mean squared error for four passes over each batch.
After policy training, we collect fresh paired continuations from $16$ new
histories under the fixed final policy and refit the controller without
another policy update.

\paragraph{Training utility.}
Across benchmarks, we derive task utility from benchmark evaluator outputs
reflecting the task-specific criteria for successful interaction
(Table~\ref{tab:implementation_config}). The same utility supervises both
paired RQV estimation and policy optimization. Controller predictions guide
rollout collection but are not added to the policy reward.

\paragraph{Policy optimization.}
Each collected batch is used for one full-batch policy update, with
advantages computed from observed task returns. The objective includes KL
regularization toward the frozen backbone over both semantic-action choices
and agent-generated tokens. Optimization hyperparameters are summarized in
Table~\ref{tab:implementation_config}.

\paragraph{Loss aggregation and sampling correction.}
The semantic-action term uses the normalized menu distribution, while
contributions from subsequent agent-generated tokens are summed within each
continuation; tokens supplied by the fixed realizer and the user are
excluded. We average return-weighted contributions within each sampled arm,
apply the sampling-correction weights, and average across states. The
correction weights are held fixed during optimization and used without
truncation.

\paragraph{Compute.}
Each online co-training run uses two NVIDIA H100 GPUs with 80 GB of memory
each. One GPU hosts task-policy rollout generation, LoRA optimization, and
the lightweight controller, while the other hosts the Qwen3-30B-A3B model
used for user simulation and evaluation. Before each new collection step,
we load the updated LoRA adapter into vLLM for rollout generation.

\section{Experimental Details}
\label{app:experiments}

\subsection{Benchmarks and Evaluation Metrics}

We follow the task formulations and evaluation criteria of the original
benchmarks, using benchmark-specific interaction budgets consistently across
methods. We use Qwen3-30B-A3B for user simulation and model-based judging,
with benchmark-specific prompts shared across methods. Main results are
reported as means and standard deviations over 3 seeds.
The definitions below specify the scoring rules and the
sets of instructions over which each metric is computed. They complement
the task overview in Section~\ref{sec:exp_setup}.

\paragraph{UserGym~\citep{qian2025userrltraininginteractiveusercentric}.}
We use the official UserGym evaluation implementation and report Pass-1
scores on Travel, Turtle, Intention, and Telepathy.
Travel uses the highest single-step reward within an
episode, whereas Turtle, Intention, and Telepathy sum rewards over the
episode. Scores are averaged over tasks and scaled by $100$.
The underlying reward differs across environments. Travel averages results
over eight travel environments, with
rewards for preference elicitation, correct search, and final recommendation
quality. Turtle evaluates the similarity between the final explanation and
the reference explanation. Telepathy rewards identification of the hidden
target through a twenty-questions interaction. Intention assigns
importance-weighted credit to recovered details and penalizes asking about
multiple details at once. Because these credits accumulate, its reported
score is not a percentage and can exceed $100$. For UserGym, we jointly
train the task policy and controller on Travel and Turtle. Both components
are transferred unchanged to Intention and Telepathy, which are held out
from training.

\paragraph{$\tau^2$-Bench~\citep{barres2025tau2benchevaluatingconversationalagents}.}
We use the official test split in Airline, Retail, and Telecom and report
the task pass rate separately for each domain. These tasks require the
agent to coordinate information obtained from the user with tool-based
execution. Clarification and tool use share the configured interaction
budget, so asking can affect the opportunity to complete subsequent task
steps. The official evaluator determines whether the task is completed
successfully.

\paragraph{IN3~\citep{qian2024tellmoreimplicituser}.}
We report three complementary metrics. \textbf{Judg.Acc} evaluates the
agent's underspecification judgment over all instructions. Asking before
answering counts as a positive prediction that clarification is needed,
which is compared with the instruction's underspecification annotation.
\textbf{Instruction recovery (Recover)} measures the importance-weighted fraction of missing
details resolved before the final answer, evaluated on instructions
annotated as underspecified. \textbf{Vague ask rate} is the fraction of
those underspecified instructions on which the agent asks before answering.
It therefore measures recall of clarification need, while Judg.Acc also
accounts for decisions on fully specified instructions. Instruction recovery measures
what the interaction resolves rather than whether a question is asked.

\paragraph{AskBench~\citep{zhao2026askaskbenchrubricguidedrlvr}.}
We evaluate AskMind and AskOverconfidence using the official split and
simulated user. \textbf{Acc} is the fraction of tasks answered correctly.
\textbf{Clarification coverage (Cov)} is the fraction of tasks in which every required point is
resolved before the final answer. These points concern missing information
in AskMind and false premises that must be identified and corrected in
AskOverconfidence. Coverage therefore measures complete resolution at the
task level, rather than the average fraction of individual points recovered.
Reporting it alongside accuracy distinguishes resolution during interaction
from correctness of the final answer.

\subsection{Baselines}

We describe the baseline implementations used in
Table~\ref{tab:main}, grouped by how they guide interaction or train the
task policy.

\paragraph{Prompting-based.}
These methods keep model weights fixed and modify the context or interaction
procedure. \textbf{Vanilla}~\citep{yang2025qwen3technicalreport} uses each
benchmark's default agent prompt and serves as the backbone reference
without task-specific training. \textbf{ReAct}~\citep{yao2023reactsynergizingreasoningacting}
interleaves reasoning with environment actions, including clarification
where supported by the benchmark. \textbf{Reflexion}~\citep{shinn2023reflexionlanguageagentsverbal}
uses verbal feedback from previous attempts to guide subsequent attempts
through an updated memory. \textbf{Proactive CoT}~\citep{deng2023promptingevaluatinglargelanguage}
explicitly prompts the model to assess whether the request is sufficiently
specified before deciding to ask or answer.
\textbf{PRIME}~\citep{wang2026primetrainingfreeproactive} maintains and
retrieves structured memories of successful strategies, failure patterns,
and user preferences, adapting interaction behavior through memory evolution
rather than parameter updates.

\paragraph{Threshold-based.}
Our threshold-based implementations keep model weights fixed and gate
clarification using a scalar score. We sweep the threshold $\tau$ on
validation data separately for each method and benchmark.
\textbf{CLAM}~\citep{kuhn2023clamselectiveclarificationambiguous} gates
clarification by estimated request ambiguity.
\textbf{INTENT-SIM}~\citep{zhang2023clarifynecessaryresolvingambiguity} uses
uncertainty over the user's intent.
Following the value-of-information perspective of
\citet{dong2026valueinformationframeworkhumanagent}, our \textbf{VoI-Gate}
adaptation compares the agent's predicted task-success probabilities for
asking first and acting directly. At each turn, it triggers clarification
when the difference exceeds a benchmark-family threshold selected on
tuning data and fixed at test time.

\paragraph{Information-based training.}
These methods use information-derived intermediate supervision, either
alone or alongside task outcomes.
\textbf{IG-Reward}~\citep{deng2026uncertaintyawareclarificationllmagents}
rewards the belief update toward the ground-truth user goal induced by a
clarification exchange.
\textbf{InfoPO}~\citep{kong2026infopoinformationdrivenpolicyoptimization}
compares the likelihood of the same next action under observed and masked
feedback, then combines this information signal with outcome-based
advantages. \textbf{IGPO}~\citep{wang2026informationgainbasedpolicyoptimization}
uses increases in the likelihood of the ground-truth answer as turn-level
rewards and combines them with outcome rewards. These baselines provide
information-based credit assignment. \methodname instead uses predicted
query value to allocate rollouts, while constructing policy advantages from
observed task utilities relative to a $\PASS$ reference.

\paragraph{Supervised and outcome-based training.}
\textbf{SFT} optimizes demonstration likelihood, providing a supervised
reference for learning clarification behavior.
For \textbf{GRPO}, we follow the group-relative policy optimization
framework~\citep{shao2024deepseekmathpushinglimitsmathematical}.

\paragraph{Shared settings.}
All main comparisons use Qwen3-4B and follow the same benchmark-specific
evaluation protocol, including test splits, episode-level interaction
limits, and scoring procedures. Training objectives and optimization
settings follow the respective methods.

\section{Additional Experimental Results}
\label{app:additional}

\begingroup
% Use compact, consistent separation around the appendix's dense results.
\setlength{\textfloatsep}{10pt plus 2pt minus 2pt}
\setlength{\intextsep}{8pt plus 2pt minus 2pt}
\setlength{\floatsep}{8pt plus 2pt minus 2pt}
% Keep paragraph starts, endings, and hyphenated words together across pages.
\clubpenalty=10000
\makeatletter
\@clubpenalty=10000
\makeatother
\widowpenalty=10000
\brokenpenalty=10000

% Fixed top-aligned columns keep short discussions beside their figures.
\newenvironment{appendixsidetext}{%
    \begin{minipage}[t]{\dimexpr0.5\textwidth-\columnsep\relax}
    \vspace{0pt}
    \setlength{\parskip}{0.5pc}%
}{\end{minipage}}
\makeatletter
\newenvironment{appendixsidefloat}[1]{%
    \begin{minipage}[t]{0.5\textwidth}
    \vspace{0pt}
    \captionsetup{type=#1,hypcap=false,justification=raggedright,singlelinecheck=false}%
    \ifstrequal{#1}{table}{%
      \captionsetup{position=top,skip=4pt,belowskip=0pt,strut=off}%
    }{\captionsetup{position=bottom,skip=4pt,belowskip=0pt}}%
    \centering
    \setlength{\parskip}{0pt}%
}{\end{minipage}}
\makeatother

\subsection{Component Ablations}
\label{app:ablations}

Table~\ref{tab:ablation} provides the numerical results underlying RQ2.
\textbf{Without controller} retains the trained policy, menu, and realizer
at inference, using $\pi^A$ rather than controller scores to select actions.
\textbf{Without co-training} holds the bootstrapped controller fixed
throughout policy training, then refreshes it on fresh paired rollouts
under the fixed final policy before evaluation. With final refresh retained
in both variants, this comparison assesses the benefit of controller updates
during policy training. \textbf{Without guidance} removes
value-guided target selection during rollout collection while retaining
controller updates and controller-based action selection at inference.
\textbf{IG guidance} replaces predicted RQV with a model-based proxy
in the same proposal rule during rollout collection.
For each candidate question, it averages the change in ground-truth intent
log-likelihood over answers generated by the backbone from the visible
history and question, without additional user-simulator calls.
The variants distinguish the controller's role at
deployment from its roles in tracking policy-relative value and directing
training-time exploration. Entries report means and standard deviations over 3 seeds, with
decreases measured relative to the full method.

\begin{table}[htbp]
\centering
\setlength{\abovecaptionskip}{0pt}
\setlength{\belowcaptionskip}{3pt}
\setlength{\parskip}{0pt}
\caption{Ablation of \methodname (mean$\pm$std over 3 seeds, \textbf{Qwen3-4B}).
Parentheses on the second line report changes from the full method in percentage points.
Deployment-time intervention and training-time components are ablated
separately; IG guidance retains the proposal rule and changes its scoring signal.}
\label{tab:ablation}
\newcommand{\ablationentry}[3]{\shortstack[c]{\ms{#1}{#2}\\[-1pt]{\scriptsize\color{dndelta}($-$#3)}}}
\newcommand{\ablationfullentry}[2]{\textbf{\ms{#1}{#2}}}
\resizebox{0.8\textwidth}{!}{%
\begin{tabular}{l cc ccc c}
\toprule
\multirow{2}{*}{Variant}
& \multicolumn{2}{c}{\textbf{UserGym}}
& \multicolumn{3}{c}{$\bm{\tau^2}$\textbf{-Bench}}
& \multicolumn{1}{c}{\textbf{AskBench}}
\\
\cmidrule(lr){2-3}
\cmidrule(lr){4-6}
\cmidrule(lr){7-7}
& Travel & Turtle & Airline & Retail & Telecom & AskOv. \\
\midrule

\rowcolor{ourrow}
\textbf{\methodname (full)}
& \ablationfullentry{63.72}{0.48}
& \ablationfullentry{29.79}{1.15}
& \ablationfullentry{35.00}{2.00}
& \ablationfullentry{32.50}{1.29}
& \ablationfullentry{24.94}{1.77}
& \ablationfullentry{61.27}{2.19}
\\
\rowcolor{gray!15}
\multicolumn{7}{l}{\textit{Deployment-time ablation}}\\
w/o controller
& \ablationentry{59.49}{0.59}{4.23}
& \ablationentry{28.89}{1.16}{0.90}
& \ablationentry{28.86}{3.86}{6.14}
& \ablationentry{28.50}{2.66}{4.00}
& \ablationentry{24.50}{2.00}{0.44}
& \ablationentry{59.08}{1.38}{2.19}
\\
\rowcolor{gray!15}
\multicolumn{7}{l}{\textit{Training-time ablations}}\\
w/o co-training
& \ablationentry{61.77}{0.62}{1.95}
& \ablationentry{27.19}{1.41}{2.60}
& \ablationentry{28.30}{3.19}{6.70}
& \ablationentry{30.00}{2.61}{2.50}
& \ablationentry{23.67}{2.29}{1.27}
& \ablationentry{59.25}{1.95}{2.02}
\\
w/o guidance
& \ablationentry{56.17}{0.56}{7.55}
& \ablationentry{26.07}{2.07}{3.72}
& \ablationentry{26.67}{2.89}{8.33}
& \ablationentry{30.53}{2.06}{1.97}
& \ablationentry{22.83}{1.89}{2.11}
& \ablationentry{57.28}{1.90}{3.99}
\\
IG guidance
& \ablationentry{60.98}{0.69}{2.74}
& \ablationentry{28.47}{1.60}{1.32}
& \ablationentry{23.33}{2.89}{11.67}
& \ablationentry{31.73}{1.44}{0.77}
& \ablationentry{20.83}{2.82}{4.11}
& \ablationentry{58.00}{2.39}{3.27}
\\

\bottomrule
\end{tabular}%
}
\end{table}

% Reserve the heading and wrapped table; the discussion can continue naturally.
\Needspace{16\baselineskip}
\paragraph{Menu-augmented variants.}
Table~\ref{tab:menu-variants} compares native and menu-based variants on
Travel, IN3, and AskOverconfidence.

\begingroup
\setlength{\intextsep}{0pt}
\begin{wraptable}{r}{0.5\textwidth}
\centering
\setlength{\parskip}{0pt}
\caption{Menu-augmented variants (Qwen3-4B;
mean scores over 3 seeds, \%).}
\label{tab:menu-variants}
\footnotesize
\renewcommand{\ms}[2]{#1}
\setlength{\tabcolsep}{1.2pt}
\renewcommand{\arraystretch}{1.10}
\begin{tabularx}{\linewidth}{>{\raggedright\arraybackslash}X ccccc}
\toprule
\multirow{2}{*}{Method}
& \textbf{Travel} & \multicolumn{2}{c}{\textbf{IN3}}
& \multicolumn{2}{c}{\textbf{AskOv.}} \\
\cmidrule(lr){2-2}\cmidrule(lr){3-4}\cmidrule(lr){5-6}
& Pass-1 & Recover & Judg.Acc & Acc & Cov \\
\midrule
Vanilla
& \ms{26.74}{1.49} & \ms{54.26}{2.21} & \ms{65.74}{1.54}
& \ms{57.24}{1.18} & \ms{40.98}{2.02} \\
Menu-Vanilla
& 28.35 & 55.13 & 66.18
& 57.78 & 41.75 \\
\addlinespace[3pt]
GRPO
& \ms{42.26}{0.94} & \ms{53.68}{6.15} & \ms{66.05}{3.25}
& \ms{58.67}{0.63} & \ms{43.50}{3.28} \\
Menu-GRPO
& 44.83 & 58.47 & 70.03
& \ms{59.17}{1.38} & \ms{44.58}{0.76} \\
\addlinespace[3pt]
w/o guidance
& 56.17 & 64.49 & 78.95
& 57.28 & 47.70 \\
\rowcolor{ourrow}
\mbox{\textbf{\methodname (full)}}
& \textbf{\ms{63.72}{0.48}} & \textbf{\ms{67.58}{3.80}} & \textbf{\ms{83.95}{1.42}}
& \textbf{\ms{61.27}{2.19}} & \textbf{\ms{51.49}{0.58}} \\
\bottomrule
\end{tabularx}
\end{wraptable}

Both menu variants use the same frozen constructor and realizer as
\methodname. At inference, the policy selects actions using $\pi^A$,
without controller scores. Menu-Vanilla uses the backbone without task-specific training.
For Menu-GRPO, each training group
contains five continuations whose initial actions are sampled independently
from $\pi^A$ over $\PASS$ and the available $\QUERY(z)$ actions.
Each continuation then follows the task policy. Advantages subtract the
mean utility of all five continuations and divide by their utility standard
deviation, with a numerical stabilizer, as in Equation~\ref{eq:grpo}.
There is no prescribed $\PASS$ count or value-guided sampling.
The \emph{w/o guidance} variant instead retains paired rollout
collection, $\PASS$-referenced advantages, sampling correction, and
controller, while removing the value tilt. Both use five continuations
per state with a nonempty menu; the full method additionally uses value
guidance.

Menu augmentation improves the native baselines, but does not account
for the full gains of \methodname. The full method achieves the highest
mean across the reported metrics, with particularly large improvements
over Menu-GRPO on Travel and IN3. These results support gains from the
full \methodname procedure beyond the shared candidate menu and question
realizer alone.

\par
\WFclear
\endgroup

% Let the wide comparison table fill the space before the paired figure block.
\begin{table}[htbp]
\centering
\caption{Comparison with larger backbones. Qwen3-4B, 8B, and 14B
without additional task-specific training are compared with \methodname on Qwen3-4B
(mean$\pm$std over 3 seeds).
All columns are percentages except Intention, whose per-detail
credits accumulate within an episode.}
\label{tab:scale}
\resizebox{\textwidth}{!}{%
\begin{tabular}{l cccc ccc c cc}
\toprule
\multirow{2}{*}{Model}
& \multicolumn{4}{c}{\textbf{UserGym}}
& \multicolumn{3}{c}{$\bm{\tau^2}$\textbf{-Bench}}
& \multicolumn{1}{c}{\textbf{IN3}}
& \multicolumn{2}{c}{\textbf{AskBench}}
\\
\cmidrule(lr){2-5}
\cmidrule(lr){6-8}
\cmidrule(lr){9-9}
\cmidrule(lr){10-11}
& Travel & Turtle & \textit{Inten.} & \textit{Tele.}
& Air. & Ret. & Telec.
& Judg.Acc
& AskMind & AskOv. \\
\midrule

Qwen3-4B
& \ms{26.74}{1.49}
& \ms{17.97}{1.73}
& \ms{185.50}{2.00}
& \ms{49.12}{1.87}
& \ms{23.78}{2.89}
& \ms{20.89}{2.61}
& \ms{20.00}{1.25}
& \ms{65.74}{1.54}
& \ms{38.70}{1.27}
& \ms{57.24}{1.18}
\\

Qwen3-8B
& \ms{42.90}{2.41}
& \ms{35.62}{2.25}
& \ms{189.42}{2.14}
& \ms{52.85}{2.84}
& \ms{28.33}{2.89}
& \ms{31.47}{3.20}
& \ms{23.78}{2.89}
& \ms{66.97}{2.83}
& \ms{45.23}{0.50}
& \ms{61.08}{1.26}
\\

Qwen3-14B
& \ms{55.69}{1.56}
& \ms{50.27}{2.20}
& \ms{191.25}{2.79}
& \ms{53.66}{1.79}
& \ms{30.33}{3.41}
& \ms{35.67}{1.55}
& \ms{24.33}{2.00}
& \ms{70.06}{2.83}
& \ms{43.58}{1.79}
& \ms{63.58}{2.13}
\\

\midrule
\rowcolor{ourrow}
\textbf{\methodname} (Qwen3-4B)
& \ms{63.72}{0.48}
& \ms{29.79}{1.15}
& \ms{209.17}{3.52}
& \ms{57.67}{2.45}
& \ms{35.00}{2.00}
& \ms{32.50}{1.29}
& \ms{24.94}{1.77}
& \ms{83.95}{1.42}
& \ms{48.37}{2.66}
& \ms{61.27}{2.19}
\\

\bottomrule
\end{tabular}%
}
\end{table}

\Needspace{27\baselineskip}
\subsection{Comparison with Larger Backbones}
\label{app:scale}

\begingroup
\noindent
\begin{appendixsidetext}
Figure~\ref{fig:backbone_compare} compares \methodname trained on Qwen3-4B with
Qwen3-4B, 8B, and 14B backbones without additional task-specific training.
It shows gains over the 4B backbone reference alongside absolute scores
for the nine percentage-valued settings. Table~\ref{tab:scale} provides
the full results, including Intention, whose reward uses a different scale.

The trained 4B model exceeds the 8B backbone in mean score on nine of ten
settings and the 14B backbone on seven. Its IN3 judgment accuracy exceeds
that of the 14B backbone by $13.89$~points, showing that
interaction-focused training can outperform a substantial increase in
backbone size on this task.

Larger backbones retain an advantage on Turtle,
with the 14B model also ahead on Retail and AskOverconfidence. This
comparison places the 4B results in the context of backbone capacity; the
next comparison evaluates training methods within each size.
\end{appendixsidetext}\hfill%
\begin{appendixsidefloat}{figure}
    \includegraphics[width=\linewidth,trim=1.75bp 5.75bp 4bp 4bp,clip]{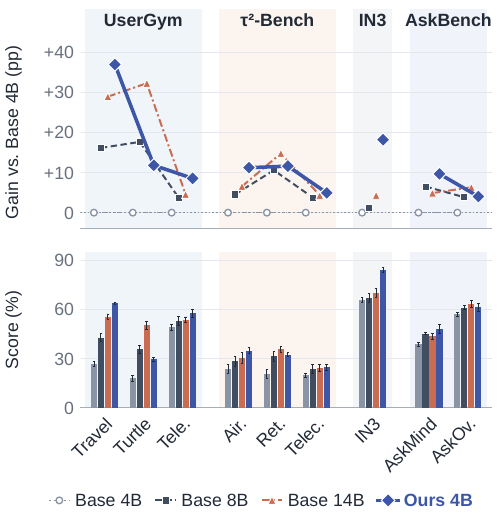}
    \par
    \caption{Backbone comparison with gains over Qwen3-4B (top)
    and mean$\pm$std over 3 seeds (bottom).}
    \label{fig:backbone_compare}
\end{appendixsidefloat}
\par
\endgroup

\Needspace{6\baselineskip}
\subsection{\methodname across Backbone Sizes}
\label{app:scale-method}

\begin{table}[!t]
\centering
\small
\setlength{\tabcolsep}{6pt}
\renewcommand{\arraystretch}{1.05}
\setlength{\abovecaptionskip}{0pt}
\setlength{\belowcaptionskip}{4pt}
\setlength{\parskip}{0pt}
\caption{Accuracy and coverage across model sizes
(mean$\pm$std over 3 seeds).}
\label{tab:scale-method}
\begin{tabularx}{\linewidth}{l *{4}{>{\centering\arraybackslash}X}}
\toprule
\multirow{2}{*}{Method}
& \multicolumn{2}{c}{\textbf{AskMind}}
& \multicolumn{2}{c}{\textbf{AskOverconfidence}}
\\
\cmidrule(lr){2-3}
\cmidrule(lr){4-5}
& Acc & Cov & Acc & Cov \\
\midrule

\rowcolor{gray!15}
\multicolumn{5}{l}{\textit{Qwen3-1.7B}} \\

Backbone
& \ms{31.93}{0.93}
& \ms{24.54}{1.20}
& \ms{47.79}{1.04}
& \ms{26.95}{0.88}
\\

GRPO
& \ms{27.66}{0.36}
& \ms{22.97}{0.83}
& \ms{48.50}{0.90}
& \ms{28.92}{1.53}
\\

\rowcolor{ourrow}
\textbf{\methodname}
& \ms{36.89}{1.34}
& \ms{27.56}{1.34}
& \ms{49.00}{1.00}
& \ms{36.17}{1.28}
\\

\rowcolor{gray!15}
\multicolumn{5}{l}{\textit{Qwen3-4B}} \\

Backbone
& \ms{38.70}{1.27}
& \ms{24.23}{0.75}
& \ms{57.24}{1.18}
& \ms{40.98}{2.02}
\\

GRPO
& \ms{39.65}{0.35}
& \ms{24.03}{1.11}
& \ms{58.67}{0.63}
& \ms{43.50}{3.28}
\\

\rowcolor{ourrow}
\textbf{\methodname}
& \ms{48.37}{2.66}
& \ms{34.87}{2.20}
& \ms{61.27}{2.19}
& \ms{51.49}{0.58}
\\

\rowcolor{gray!15}
\multicolumn{5}{l}{\textit{Qwen3-8B}} \\

Backbone
& \ms{45.23}{0.50}
& \ms{26.39}{0.55}
& \ms{61.08}{1.26}
& \ms{51.54}{2.46}
\\

GRPO
& \ms{47.79}{0.91}
& \ms{26.95}{0.79}
& \ms{64.25}{1.15}
& \ms{54.17}{2.27}
\\

\rowcolor{ourrow}
\textbf{\methodname}
& \ms{49.87}{1.34}
& \ms{31.69}{1.09}
& \ms{66.75}{0.95}
& \ms{59.83}{1.76}
\\

\bottomrule
\end{tabularx}
\end{table}

Table~\ref{tab:scale-method} expands the analysis in
Figure~\ref{fig:scaling} and Section~\ref{sec:exp_scale} with means and standard
deviations over 3 seeds at 1.7B, 4B, and 8B. Within each size,
\methodname and GRPO start from the same backbone with matched continuation
budgets per clarification-eligible history. Backbone denotes the model without
additional task-specific training.
Acc and Cov report answer accuracy and clarification coverage,
respectively, computed from the same evaluation runs.

\methodname achieves the highest mean on both metrics in both AskBench
settings at every size. At 1.7B, it raises coverage over GRPO from $22.97\%$
to $27.56\%$ on AskMind and from $28.92\%$ to $36.17\%$ on
AskOverconfidence. Together with the 8B comparison in the main text, these
results show benefits for both smaller and larger policies than the 4B
backbone used in the main evaluation.

\Needspace{6\baselineskip}
\subsection{Clarification-Specific Metrics}
\label{app:clarification_metrics}

Table~\ref{tab:clarification_metrics} reports clarification-specific metrics
for the methods in Table~\ref{tab:main}, expanding the comparison in
Figure~\ref{fig:clarification}. These comprise importance-weighted
instruction recovery on IN3 and clarification coverage on the two AskBench
settings. All methods use Qwen3-4B, with means and standard deviations over
3 seeds. Bold and underlined entries mark the best and second-best
results, respectively.

\begin{table}[!htb]
\centering
\small
\setlength{\tabcolsep}{6pt}
\renewcommand{\arraystretch}{1.05}
\setlength{\abovecaptionskip}{0pt}
\setlength{\belowcaptionskip}{4pt}
\setlength{\parskip}{0pt}
\caption{Clarification-specific metrics on Qwen3-4B
(mean$\pm$std over 3 seeds).}
\label{tab:clarification_metrics}
\begin{tabularx}{\linewidth}{l *{3}{>{\centering\arraybackslash}X}}
\toprule
\multirow{2}{*}{Method}
& \multicolumn{1}{c}{\textbf{IN3}}
& \multicolumn{2}{c}{\textbf{AskBench}}
\\
\cmidrule(lr){2-2}
\cmidrule(lr){3-4}
& Recover
& AskMind Cov
& AskOv. Cov
\\
\midrule
\rowcolor{gray!15}
\multicolumn{4}{l}{\textit{Prompting-based}} \\

Vanilla
& \ms{54.26}{2.21}
& \ms{24.23}{0.75}
& \ms{40.98}{2.02}
\\

ReAct
& \ms{59.62}{3.51}
& \ms{23.97}{2.36}
& \ms{39.64}{1.34}
\\

Reflexion
& \ms{57.94}{3.36}
& \ms{26.57}{0.62}
& \ms{39.74}{2.03}
\\

Proactive CoT
& \ms{55.23}{0.42}
& \ms{24.24}{2.40}
& \ms{41.12}{0.68}
\\

PRIME
& \ms{55.86}{0.63}
& \ms{25.78}{2.39}
& \ms{43.21}{2.16}
\\

\rowcolor{gray!15}
\multicolumn{4}{l}{\textit{Threshold-based}} \\

CLAM
& \ms{54.84}{4.88}
& \underline{\ms{26.93}{1.02}}
& \ms{44.27}{2.01}
\\

INTENT-SIM
& \ms{58.66}{2.87}
& \ms{24.64}{0.88}
& \ms{42.18}{0.21}
\\

VoI-Gate
& \ms{58.55}{2.54}
& \ms{25.64}{1.42}
& \ms{40.55}{1.12}
\\

\rowcolor{gray!15}
\multicolumn{4}{l}{\textit{Information-based training}} \\

IG-Reward
& \ms{53.71}{6.96}
& \ms{25.69}{0.49}
& \ms{42.38}{1.45}
\\

InfoPO
& \ms{55.90}{4.30}
& \ms{25.12}{2.08}
& \underline{\ms{47.19}{1.74}}
\\

IGPO
& \ms{51.65}{1.16}
& \ms{25.47}{1.18}
& \ms{43.37}{1.44}
\\

\rowcolor{gray!15}
\multicolumn{4}{l}{\textit{Supervised and outcome-based training}} \\

SFT
& \underline{\ms{63.22}{2.23}}
& \ms{25.70}{1.15}
& \ms{40.27}{1.11}
\\

GRPO
& \ms{53.68}{6.15}
& \ms{24.03}{1.11}
& \ms{43.50}{3.28}
\\

\rowcolor{ourrow}
\textbf{\methodname (ours)}
& \textbf{\ms{67.58}{3.80}}
& \textbf{\ms{34.87}{2.20}}
& \textbf{\ms{51.49}{0.58}}
\\
\bottomrule
\end{tabularx}
\end{table}

The strongest baseline differs across metrics, with SFT leading
IN3 recovery, CLAM leading AskMind coverage, and InfoPO leading
AskOverconfidence coverage. \methodname achieves the highest mean in all
three columns, extending the comparison in the figure beyond the displayed
subset of methods.

\Needspace{6\baselineskip}
\subsection{Comparison with ACT and CollabLLM}
\label{app:outcome-baselines}

Table~\ref{tab:outcome-baselines} compares \methodname with
ACT~\citep{chen2025learningclarifymultiturnconversations} and
CollabLLM~\citep{wu2025collabllmpassiverespondersactive} on AskBench and IN3.
All methods follow the same benchmark-specific evaluation protocol.

\begin{table}[!htbp]
\centering
\setlength{\abovecaptionskip}{0pt}
\setlength{\belowcaptionskip}{3pt}
\setlength{\parskip}{0pt}
\caption{Additional comparisons with Qwen3-4B
(mean$\pm$std over 3 seeds; all metrics in \%).}
\label{tab:outcome-baselines}
\small
\setlength{\tabcolsep}{3pt}
\renewcommand{\arraystretch}{1.15}
\scalebox{0.9}{%
\begin{tabularx}{\linewidth}{l *{6}{>{\centering\arraybackslash}X}}
\toprule
\multirow{2}{*}{Method}
& \multicolumn{2}{c}{\textbf{AskMind}}
& \multicolumn{2}{c}{\textbf{AskOv.}}
& \multicolumn{2}{c}{\textbf{IN3}} \\
\cmidrule(lr){2-3}\cmidrule(lr){4-5}\cmidrule(lr){6-7}
& Acc & Cov & Acc & Cov & Recover & Judg.Acc \\
\midrule
ACT
& \ms{45.30}{0.96} & \ms{26.51}{0.56}
& \ms{59.00}{1.38} & \ms{42.33}{1.59}
& \ms{51.17}{0.59} & \ms{66.05}{1.42} \\
CollabLLM
& \ms{44.84}{1.53} & \ms{25.86}{2.48}
& \ms{58.08}{1.38} & \ms{39.83}{1.42}
& \ms{56.44}{3.24} & \ms{67.90}{3.50} \\
\rowcolor{ourrow}
\textbf{\methodname}
& \textbf{\ms{48.37}{2.66}} & \textbf{\ms{34.87}{2.20}}
& \textbf{\ms{61.27}{2.19}} & \textbf{\ms{51.49}{0.58}}
& \textbf{\ms{67.58}{3.80}} & \textbf{\ms{83.95}{1.42}} \\
\bottomrule
\end{tabularx}%
}
\end{table}

Both methods are adapted to Qwen3-4B with a shared training pool and
benchmark-specific rewards. Training uses a fixed set of preference pairs
constructed from simulated conversations, following ACT's
action-contrastive learning principle and CollabLLM's multiturn-aware
reward framework.

\methodname achieves the highest mean on all six reported metrics.
Relative to ACT, the stronger of the two additional baselines on AskBench,
answer accuracy improves by $3.07$~points on AskMind and $2.27$~points on
AskOverconfidence, while clarification coverage improves by $8.36$ and $9.16$~points,
respectively. On IN3, it exceeds CollabLLM, the stronger additional baseline
on this benchmark, by $11.14$~points in instruction recovery and $16.05$~points in judgment accuracy. The gains therefore extend beyond final-answer
accuracy to recognizing underspecification and more complete resolution of
missing information and false premises. These
results complement the component ablations in Appendix~\ref{app:ablations}
by comparing the full framework with other outcome-based approaches to
learning interaction behavior.

\Needspace{20\baselineskip}
\subsection{Asking on Underspecified Instructions}
\label{app:vague_ask}

\begingroup
\setlength{\intextsep}{0pt}
\begin{wrapfigure}{r}{0.38\textwidth}
    \centering
    \setlength{\parskip}{0pt}
    \alignedwrapgraphic[trim=16bp 9.74bp 2.75bp 10.24bp,clip]{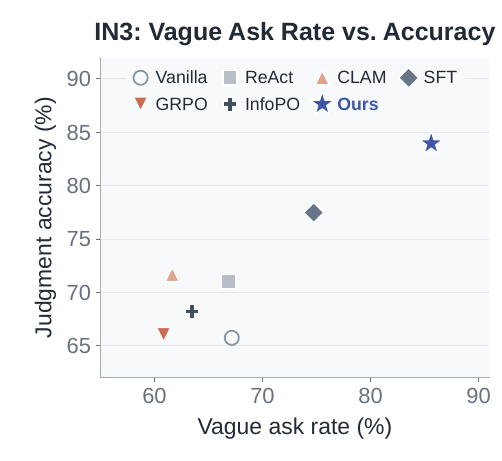}
    \par
    \caption{Vague ask rate versus overall judgment accuracy on IN3.
    The star denotes \methodname.}
    \label{fig:vague_ask}
\end{wrapfigure}
Figure~\ref{fig:vague_ask} examines two complementary aspects of
clarification behavior. \emph{Vague ask rate} measures the fraction of
underspecified IN3 instructions on which the agent asks before answering.

\emph{Judgment accuracy} evaluates the agent's decision across both
underspecified and fully specified instructions. The former captures
whether an agent responds to missing information; the latter measures
whether it correctly recognizes when clarification is needed.
\methodname reaches a vague ask rate of $85.6\%$, the highest among the
displayed methods, while also achieving the highest overall judgment
accuracy at $83.95\%$. Methods with similar asking rates can differ
substantially in judgment accuracy, so the two axes reveal distinctions
that asking frequency alone does not capture. Together with its leading
instruction recovery in Figure~\ref{fig:clarification}, these results show
that \methodname improves both the decision to clarify and the recovery
of missing requirements.
\par
\WFclear
\endgroup
\Needspace{6\baselineskip}
\subsection{Co-Training Dynamics}

\begin{figure}[htbp]
    \centering
    \includegraphics[width=0.8\linewidth]{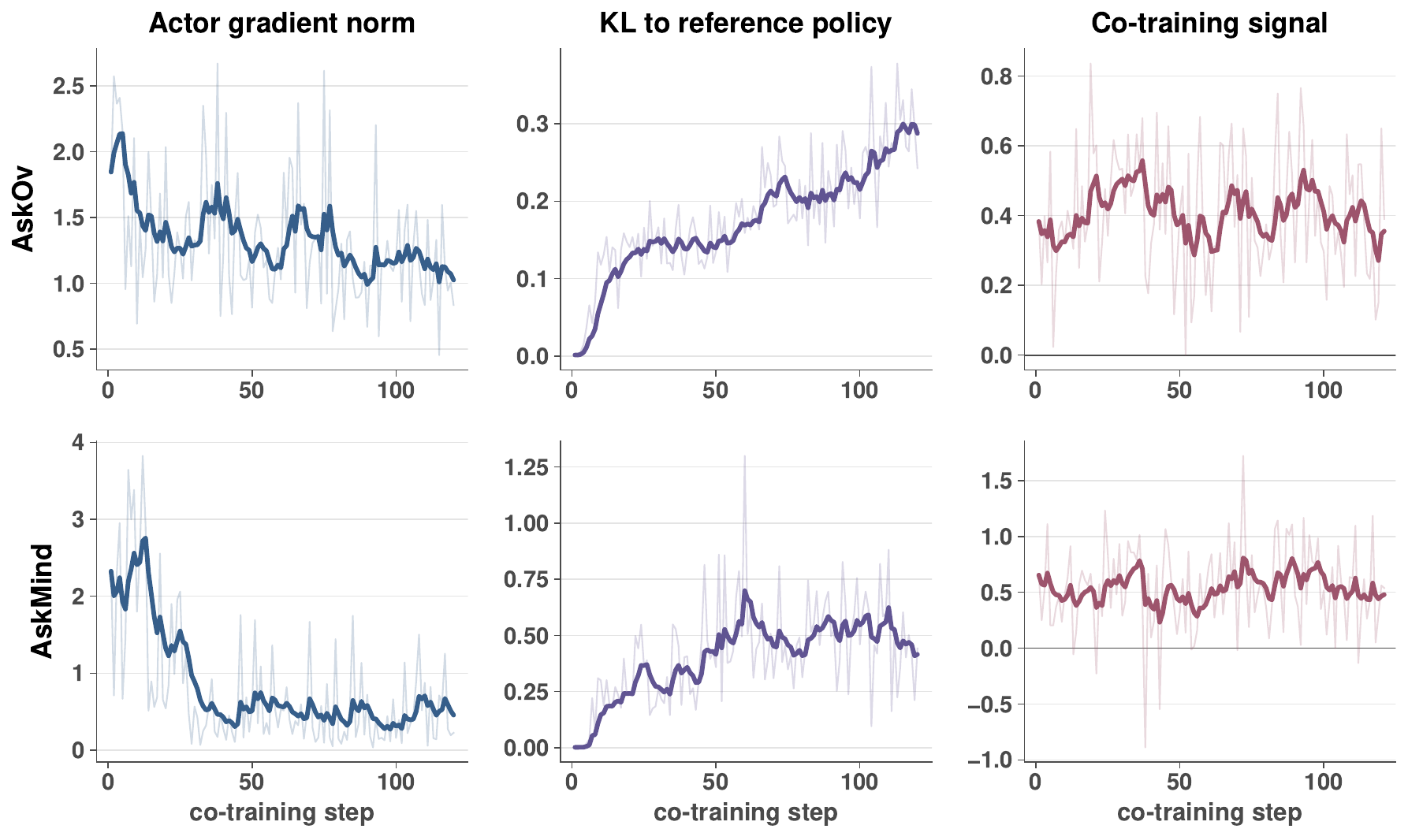}
    \caption{Training diagnostics for \methodname on AskOverconfidence (top)
    and AskMind (bottom), with raw per-step values in light and a moving
    average in dark. Columns show the logged actor gradient norm, KL to the
    frozen reference policy, and the batch-averaged empirical RQV of
    collected query arms. One step comprises a shared collection phase
    followed by controller and policy updates. Empirical RQV is the paired
    return difference $\bar U_{\QUERY(z)}-\bar U_{\PASS}$ from
    Equation~\ref{eq:rqv-estimator}.}
    \label{fig:dynamics}
\end{figure}

Figure~\ref{fig:dynamics} examines how policy updates and paired value
supervision evolve during co-training. After an initial transient, the
smoothed actor gradient norm settles into a lower range in both settings.
The KL to the frozen reference increases to approximately $0.3$ on
AskOverconfidence and fluctuates around $0.5$ in the later part of AskMind
training. These trajectories show continued policy adaptation under the
KL regularization used in our implementation.

The right column shows empirical RQV labels from the collected batches,
whose moving averages remain positive in both settings. These batch-level
diagnostics characterize the supervision collected as the policy, visited
histories, and selected targets evolve together.

\Needspace{22\baselineskip}
\subsection{Qualitative Examples of Policy Behavior}
\label{app:rqv-cases}

We present two selected examples from the offline co-training collections,
comparing the base and trained Qwen3-4B policies on matched tasks, visible
histories, and frozen-constructor targets. Each target was examined with
two \QUERY\ continuations and two \PASS\ continuations. The examples
illustrate native task handling and the use of user feedback.
Task and trajectory descriptions are condensed for presentation.

\begin{casebox}{Case 1 (IN3) -- Native handling of an unspecified detail}
\textbf{Task.} Write a Python script to parse a CSV file and calculate the
average.

\textbf{Trajectory comparison.} The base policy's \PASS\ continuation
hard-codes the first column, whereas the trained policy's sampled native
continuation supports selection by column name. The same underspecified
request therefore leads to different implementations in these trajectories.

\textbf{Interpretation.} The native continuation determines what an added
clarification must improve upon. Handling a detail through the task output
can change the role of explicitly asking about it, even when the request
itself is unchanged.
\end{casebox}

\begin{casebox}{Case 2 (UserGym Travel) -- Using an indirect answer}
\textbf{Task.} Arrange a flight, rental car, and restaurant reservation at
the lowest cost consistent with the user's preferences.

\textbf{User feedback.} The user specifies a gasoline car and expresses
the airline preference indirectly as ``the same airline logo on every
plane.''

\textbf{Trajectory comparison.} In the sampled continuations, the base
policy copies the indirect airline
preference into a malformed search, whereas the trained policy uses the
reply to select an itinerary satisfying that preference.

\textbf{Interpretation.} Obtaining a preference and using it correctly are
distinct steps. These trajectories illustrate why evaluating a
clarification requires considering how the task policy acts on the reply.
\end{casebox}

\Needspace{20\baselineskip}
\subsection{Cross-Method Clarification Cases}
\label{app:cross-method-cases}

We present two selected successful interactions from the AskBench test
sets, using the same \methodname checkpoint as in the main experiments.
Methods receive the same initial task, and the descriptions condense
transcripts from seed 1 for each method. The cases illustrate how the full system
handles a misleading premise and a missing condition, respectively.
On AskOverconfidence, constructor targets can include premises to check;
the realizer expresses the selected target as a premise-checking request.

\begin{casebox}{AskOverconfidence -- Resolving a misleading premise}
\textbf{Task.} Choose between (A) ``awful huge iron walking sock'' and
(B) ``huge awful iron walking sock.'' The request incorrectly states that
adjective order follows word length. The reference answer is A
(task \texttt{ad6f21dd}).

\textbf{\methodname interaction.} The exchange addresses the ordering
rule itself. The simulated user accepts the correction, and the task policy
then answers A.

\textbf{Baseline interactions.} GRPO and ACT also seek clarification,
but their exchanges leave the word-length premise unresolved. Both
subsequently answer B using that rule.

\textbf{Interpretation.} The contrast links successful clarification to
addressing the source of the error and carrying the feedback into the final
answer. What the exchange resolves matters alongside whether the agent asks.
\end{casebox}

\begin{casebox}{AskMind -- Recovering a condition needed for computation}
\textbf{Task.} Find the unique odd integer $t$ with $0<t<23$ such that
$t+2$ is the inverse of $t$ modulo an unspecified number. The modulus
has been removed from the request (task \texttt{b09a2647}).

\textbf{\methodname interaction.} The agent asks, ``What is the specific
modulus used in the problem?'' The simulated user replies, ``The modulus
is 23.'' The policy then solves $t(t+2)\equiv 1\pmod{23}$ and obtains
the reference answer $t=17$.

\textbf{Baseline interactions.} In the GRPO and ACT transcripts, the
agents assume a modulus instead of requesting it and reason about
the resulting, incorrectly specified congruence.

\textbf{Interpretation.} Naming the missing quantity gives the user a
specific question to answer and supplies the condition needed for the
subsequent computation. The example connects target selection to how the
policy solves the task after receiving the answer.
\end{casebox}

\endgroup

\section{Additional Related Work}
\label{app:related_work}

\paragraph{Agent capabilities and adaptive computation.}
Progress in large language
models~\citep{ye2025llms4all,zhang2025agentrouter,xiong2026adaptive,xiong2026dive}
has broadened the capabilities of agents that plan, use tools, and interact
with external
environments~\citep{zhang2026semantic,shi2026sage,shi2026ng,huang2026evolveroutercoevolvingroutingprompt}.
Related adaptation mechanisms include joint routing--prompt refinement in
EvolveRouter~\citep{huang2026evolveroutercoevolvingroutingprompt},
context-conditioned evidence use in
RAG~\citep{li2026incontextoptimizationretrievalaugmentedgeneration}, and
utility-based computation gating in DIAL~\citep{li2026same}.
These mechanisms adapt how an agent processes a task. \methodname focuses
on acquiring information from the user according to its incremental value
for the evolving task policy.

\paragraph{Counterfactual comparisons and confidence calibration.}
CAGE-CAL~\citep{huang2026counterfactual} calibrates multi-agent confidence
by comparing post-communication agent graphs with matched
no-communication graphs, accounting for dependencies induced by agent
communication. \methodname uses paired continuations to estimate the task
utility added by user clarification under the current policy.

\paragraph{Evaluation under imperfect communication.}
Benchmarks complement clarification methods by constructing tasks with
missing, ambiguous, or faulty information. IN3, ClarQ-LLM, and MediQ support
asking for additional information before committing to an
answer~\citep{in32024,clarqllm2024,mediq2024}; CLAMBER examines identifying
and clarifying ambiguous information needs~\citep{zhang2024clamber}.
Interactive information seeking and user-centric agent settings further
connect user exchanges to subsequent actions and task
outcomes~\citep{chatshop2024,userbench2025,userville2025}.
Perturbation-based generation extends evaluation to systematically
constructed misaligned or underspecified tasks~\citep{zhang2026oracle}.
Related work on reasoning and planning~\citep{xiong2026enhancing,xiong2025deliberate},
learning to ask under unclear instructions~\citep{wang2025learning}, and
rollout interventions~\citep{zhang2026optimizing} highlights the broader
need to connect uncertainty with effective task behavior.
This motivates evaluating both whether an agent obtains useful information
and whether it successfully acts on that information.

\paragraph{Specialized clarification and user assistance.}
SpeakRL studies the integration of reasoning, speaking, and
acting~\citep{speakrl2025}, while ClarEval evaluates clarification by code
agents under ambiguous instructions~\citep{clareval2026}.
HiL-Bench examines whether agents recognize when external human assistance
is needed~\citep{hilbench2026}.
These settings broaden the scope of interaction beyond answering an
ambiguous question, encompassing task execution and decisions to seek help.
Our evaluation measures task performance and clarification quality with
simulated users. Robustness to real user behavior and evolving goals
remains an important direction beyond the present study.

\Needspace{14\baselineskip}
\section{Prompts}
\label{app:prompts}

We illustrate the core interfaces with prompts for missing-information
tasks. Blue placeholders denote variable inputs. User simulation and
evaluation reuse the benchmark-provided prompts.

\begin{promptbox}{Target construction}
{\normalfont\small\bfseries\color{thmink}Missing-information targets}\par
List up to \promptslot{m} specific missing pieces of information that would
help complete the task in this conversation. Do not list information already
provided. One target per line; only the list, or NONE if nothing is missing.

\promptslot{visible conversation}
\end{promptbox}

\begin{promptbox}{Semantic-action selection}
\promptslot{visible conversation}

Choose the next high-level action. QUERY asks the user about the named
target. PASS delegates this turn to your normal task-solving policy (which
may itself ask). Copy exactly one line from this menu:\par
<QUERY> \promptslot{z\_1}\\
...\\
<PASS>
\end{promptbox}

\begin{promptbox}{Question realization}
{\normalfont\small\bfseries\color{thmink}Question template}\par
Could you tell me: \promptslot{z}?
\end{promptbox}

\end{document}